\documentclass[english]{article}
\usepackage[T1]{fontenc}
\usepackage[latin9]{inputenc}
\usepackage{geometry}
\usepackage{babel}
\usepackage{float}
\usepackage{mathtools}
\usepackage{bm}
\usepackage{amsmath}
\usepackage{amssymb}

\usepackage[unicode=true,
 bookmarks=false,
 breaklinks=false,pdfborder={0 0 1},colorlinks=false]
 {hyperref}
\hypersetup{
 colorlinks,linkcolor=red,anchorcolor=blue,citecolor=blue}

\makeatletter

\floatstyle{ruled}
\newfloat{algorithm}{tbp}{loa}
\providecommand{\algorithmname}{Algorithm}
\floatname{algorithm}{\protect\algorithmname}

\usepackage{babel}

\usepackage[T1]{fontenc}
\usepackage{cite}\usepackage{amsthm}\usepackage{dsfont}\usepackage{array}\usepackage{mathrsfs}\usepackage{comment}\onecolumn

\usepackage{color}\usepackage{babel}

\allowdisplaybreaks

\usepackage{enumitem}
\setlist[itemize]{leftmargin=1.5em}
\setlist[enumerate]{leftmargin=1.5em}

\usepackage{babel}
\usepackage{algorithm}
\usepackage{algorithmic}
\usepackage{arydshln}
\usepackage{caption}

\DeclareMathOperator{\ind}{\mathds{1}}  

\numberwithin{equation}{section}

\definecolor{yxc}{RGB}{255,0,0}
\definecolor{yjc}{RGB}{125,0,0}
\definecolor{cm}{RGB}{0,0,200}
\definecolor{kzw}{RGB}{0,150,0}
\definecolor{byw}{RGB}{0,0,200}

\makeatother

\begin{document}
\theoremstyle{plain} \newtheorem{lemma}{\textbf{Lemma}} \newtheorem{proposition}{\textbf{Proposition}}\newtheorem{theorem}{\textbf{Theorem}}\setcounter{theorem}{0}
\newtheorem{corollary}{\textbf{Corollary}} \newtheorem{assumption}{\textbf{Assumption}}
\newtheorem{example}{\textbf{Example}} \newtheorem{definition}{\textbf{Definition}}
\newtheorem{fact}{\textbf{Fact}}\newtheorem{property}{Property}
\theoremstyle{definition}

\theoremstyle{remark}\newtheorem{remark}{\textbf{Remark}}\newtheorem{condition}{Condition}\newtheorem{claim}{Claim}\newtheorem{conjecture}{Conjecture}
\title{Sample-Efficient Reinforcement Learning for Linearly-Parameterized MDPs with a Generative Model}
\author{Bingyan Wang\footnote{The first two authors contributed equally.} \thanks{Department of Operations Research and Financial Engineering, Princeton
University, Princeton, NJ 08544, USA; Email: \texttt{\{bingyanw,
yulingy,jqfan\}@princeton.edu}.} \and Yuling Yan\footnotemark[1] \footnotemark[2] \and Jianqing Fan\footnotemark[2]}

\maketitle
\begin{abstract}
The curse of dimensionality is a widely known issue in reinforcement
learning (RL). In the tabular setting where the state space $\mathcal{S}$
and the action space $\mathcal{A}$ are both finite, to obtain a nearly
optimal policy with sampling access to a generative model, the minimax optimal
sample complexity scales linearly with $|\mathcal{S}|\times|\mathcal{A}|$, which can be prohibitively large when $\mathcal{S}$
or $\mathcal{A}$ is large. This paper considers a Markov decision
process (MDP) that admits a set of state-action features, which can
linearly express (or approximate) its probability transition kernel.
We show that a model-based approach (resp.~Q-learning) provably learns
an $\varepsilon$-optimal policy (resp.~Q-function) with high probability
as soon as the sample size exceeds the order of $\frac{K}{(1-\gamma)^{3}\varepsilon^{2}}$
(resp.~$\frac{K}{(1-\gamma)^{4}\varepsilon^{2}}$), up to some logarithmic factor. Here $K$ is the feature dimension and $\gamma\in(0,1)$ is
the discount factor of the MDP. Both sample complexity bounds are
provably tight, and our result for the model-based approach matches
the minimax lower bound. Our results show that for arbitrarily large-scale
MDP, both the model-based approach and Q-learning are sample-efficient
when $K$ is relatively small, and hence the title of this paper.

\end{abstract}



\section{Introduction}

Reinforcement learning (RL) studies the problem of learning and decision
making in a Markov decision process (MDP). Recent years have seen
exciting progress in applications of RL in real world decision-making
problems such as AlphaGo \cite{silver2016mastering,silver2017mastering}
and autonomous driving \cite{kiran2021deep}. Specifically,
the goal of RL is to search for an optimal policy that maximizes the
cumulative reward, based on sequential noisy data. There are two popular
approaches to RL: model-based and  model-free ones. 
\begin{itemize}[leftmargin=*]
\item The model-based approaches start with formulating an empirical MDP
by learning the probability transition model from the collected data
samples, and then estimating the optimal policy / value function based
on the empirical MDP. 
\item The model-free approaches (e.g.~Q-learning) learn the optimal policy
or the optimal (action-)value function from samples. As its name suggests,
model-free approaches do not attempt to learn the model explicitly.
\end{itemize}
Generally speaking, model-based approaches enjoy great flexibility
since after the transition model is learned in the first place, it
can then be applied to any other problems without touching the raw
data samples. In comparison, model-free methods, due to its online
nature, are usually memory-efficient and can interact with the environment
and update the estimate on the fly. 

This paper is devoted to investigating the sample efficiency of both
model-based RL and Q-learning (arguably one of the most
commonly adopted model-free RL algorithms). It is well known that MDPs suffer from the
curse of dimensionality. For example, in the tabular setting
where the state space $\mathcal{S}$ and the action space $\mathcal{A}$
are both finite, to obtain a near optimal policy or value function
given sampling access to a generative model, the minimax optimal sample
complexity scales linearly with $|\mathcal{S}|\times|\mathcal{A}|$
\cite{azar2013minimax,agarwal2020model}. However contemporary applications of RL often
encounters environments with exceedingly large state and action spaces,
whilst the data collection might be expensive or even high-stake.
This suggests a large gap between the theoretical findings and practical
decision-making problems where $\vert\mathcal{S}\vert$ and $\vert\mathcal{A}\vert$
are large or even infinite. 

To close the aforementioned theory-practice gap, one natural idea
is to impose certain structural assumption on the MDP. In this paper
we follow the feature-based linear transition model studied in \cite{yang2019sample},
where each state-action pair $(s,a)\in\mathcal{S}\times\mathcal{A}$
admits a $K$ dimensional feature vector $\phi(s,a)\in\mathbb{R}^{K}$
that expresses the transition dynamics $\mathbb{P}(\cdot|s,a)=\Psi\phi(s,a)$
for some unknown matrix $\Psi\in\mathbb{R}^{\vert\mathcal{S}\vert\times K}$
which is common for all $(s,a)$.  This model encompasses both the tabular case and the homogeneous model in which the state space can be partitioned into $K$ equivalent classes.  Assuming access to a generative
model \cite{kakade2003sample,kearns1999finite}, under this structural
assumption, this paper aims to answer the following two questions:

\begin{itemize}[leftmargin=*]

\item[] \emph{How many samples are needed for model-based RL and
Q-learning to learn an optimal policy under the feature-based linear
transition model? }

\end{itemize}

\noindent In what follows, we will show that the answer to this question
scales linearly with the dimension of the feature space $K$ and is
independent of $\vert\mathcal{S}\vert$ and $\vert\mathcal{A}\vert$
under the feature-based linear transition model. With the aid of this
structural assumption, model-based RL and Q-learning becomes significantly
more sample-efficient than that in the tabular setting.

\paragraph{Our contributions.}

We focus our attention on an infinite horizon MDP with discount factor
$\gamma\in(0,1)$. We use $\varepsilon$-optimal policy to indicate
the policy whose expected discounted cumulative rewards are $\varepsilon$
close to the optimal value of the MDP. Our contributions are two-fold:
\begin{itemize}[leftmargin=*]
\item We demonstrate that model-based RL provably learns an $\varepsilon$-optimal
policy by performing planning based on an empirical MDP constructed
from a total number of 
\[
\widetilde{O}\left(\frac{K}{\left(1-\gamma\right)^{3}\varepsilon^{2}}\right)
\]
samples, for all $\varepsilon\in(0,(1-\gamma)^{-1/2}]$. Here $\widetilde{O}(\cdot)$
hides logarithmic factors compared to the usual $O(\cdot)$ notation.
To the best of our knowledge, this is the first theoretical guarantee
for model-based RL under the feature-based linear transition model.
This sample complexity bound matches the minimax limit established
in \cite{yang2019sample} up to logarithmic factor.
\item We also show that Q-learning provably finds an entrywise $\varepsilon$-optimal
Q-function using a total number of
\[
\widetilde{O}\left(\frac{K}{\left(1-\gamma\right)^{4}\varepsilon^{2}}\right)
\]
samples, for all $\varepsilon\in(0,1]$. This sample complexity upper
bound improves the state-of-the-art result in \cite{yang2019sample}
and the dependency on the effective horizon $(1-\gamma)^{-4}$ is
sharp in view of \cite{li2021qlearning}.
\end{itemize}
These results taken collectively show the minimax optimality of model-based
RL and the sub-optimality of Q-learning in sample complexity.

\section{Problem formulation \label{sec:Problem-formulation} }

This paper focuses on tabular MDPs in the discounted infinite-horizon
setting \cite{bertsekas2000dynamic}. Here and throughout, $\Delta_{d-1}\coloneqq\{v\in\mathbb{R}^{d}:\sum_{i=1}^{d}v_{i}=1,v_{i}\geq0,\forall i\in[d]\}$
stands for the $d$-dimensional probability simplex and $[N]\coloneqq\{1,2,\cdots,N\}$
for any $N\in\mathbb{N}^{+}$.

\paragraph{Discounted infinite-horizon MDPs.}

Denote a discounted infinite-horizon MDP by a tuple $M=(\mathcal{S},\mathcal{A},P,r,\gamma)$,
where $\mathcal{S}=\{1,\cdots,|\mathcal{S}|\}$ is a finite set of
states, $\mathcal{A}=\{1,\cdots,|\mathcal{A}|\}$ is a finite set
of actions, $P:\mathcal{S}\times\mathcal{A}\rightarrow\Delta_{\vert\mathcal{S}\vert-1}$
represents the probability transition kernel where $P(s'|s,a)$ denotes
the probability of transiting from state $s$ to state $s'$ when
action $a$ is taken, $r:\mathcal{S}\times\mathcal{A}\rightarrow[0,1]$
denotes the reward function where $r(s,a)$ is the instantaneous reward
received when taking action $a\in\mathcal{A}$ while in state $s\in\mathcal{S}$,
and $\gamma\in(0,1)$ is the discount factor. 

\paragraph{Value function and Q-function.}

Recall that the goal of RL is to learn a policy that maximizes the
cumulative reward, which corresponds to value functions or Q-functions
in the corresponding MDP. For a deterministic policy $\pi:\mathcal{S}\rightarrow\mathcal{A}$
and a starting state $s\in\mathcal{S}$, we define the value function
as
\[
V^{\pi}\left(s\right)\coloneqq\mathbb{E}\left[\sum_{k=0}^{\infty}\gamma^{k}r\left(s_{k},a_{k}\right)\,\Big|\,s_{0}=s\right]
\]
for all $s\in\mathcal{S}$. Here, the trajectory is generated by $a_{k}=\pi(s_{k})$
and $s_{k+1}\sim P(s_{k+1}|s_{k},a_{k})$ for every $k\geq0$. This
function measures the expected discounted cumulative reward received
on the trajectory $\{(s_{k},a_{k})\}_{k\geq0}$ and the expectation
is taken with respect to the randomness of the transitions $s_{k+1}\sim P(\cdot|s_{k},a_{k})$
on the trajectory. Recall that the immediate rewards lie in $[0,1]$,
it is easy to derive that $0\leq V^{\pi}(s)\leq\frac{1}{1-\gamma}$
for any policy $\pi$ and state $s$. Accordingly, we define the Q-function
for policy $\pi$ as
\[
Q^{\pi}\left(s,a\right)\coloneqq\mathbb{E}\left[\sum_{k=0}^{\infty}\gamma^{k}r\left(s_{k},a_{k}\right)|s_{0}=s,a_{0}=a\right]
\]
for all $(s,a)\in\mathcal{S}\times\mathcal{A}$. Here, the actions
are chosen by the policy $\pi$ except for the initial state (i.e.
$a_{k}=\pi(s_{k})$ for all $k\geq1$). Similar to the value function,
we can easily check that $0\leq Q^{\pi}(s,a)\leq\frac{1}{1-\gamma}$
for any $\pi$ and $(s,a)$. To maximize the value function or Q function,
previous literature \cite{bellman1959functional,sutton2018reinforcement}
establishes that there exists an optimal policy $\pi^{\star}$ which
simultaneously maximizes $V^{\pi}(s)$ (resp. $Q^{\pi}(s,a)$) for
all $s\in\mathcal{S}$ (resp. $(s,a)\in\mathcal{S}\times\mathcal{A}$).
We define the optimal value function $V^{\star}$ and optimal Q-function
$Q^{\star}$ respectively as 
\[
V^{\star}\left(s\right)\coloneqq\max_{\pi}V^{\pi}\left(s\right)=V^{\pi^{\star}}\left(s\right),\qquad Q^{\star}\left(s,a\right)\coloneqq\max_{\pi}Q^{\pi}\left(s,a\right)=Q^{\pi^\star}\left(s,a\right)
\]
for any state-action pair $(s,a)\in\mathcal{S}\times\mathcal{A}$.

\paragraph{Linear transition model.}

Given a set of $K$ feature functions $\phi_{1}$, $\phi_{2}$, $\cdots$,
$\phi_{K}$: $\mathcal{S}\times\mathcal{A}\rightarrow\mathbb{R}$,
we define $\phi$ to be a feature mapping from $\mathcal{S}\times\mathcal{A}$
to $\mathbb{R}^{K}$ such that   
\[
\phi\left(s,a\right)=\left[\phi_{1}\left(s,a\right),\cdots,\phi_{K}\left(s,a\right)\right]\in\mathbb{R}^{K}.
\]
Then we are ready to define the linear transition model \cite{yang2019sample}
as follows.

\begin{definition}[Linear transition model]\label{def:linear-model}
Given a discounted infinite-horizon MDP $M=\left(\mathcal{S},\mathcal{A},P,r,\gamma\right)$
and a feature mapping $\phi:\mathcal{S}\times\mathcal{A}\rightarrow\mathbb{R}^{K}$,
$M$ admits the linear transition model if there exists some (unknown)
functions $\psi_{1}$, $\cdots$, $\psi_{K}:\mathcal{S}\rightarrow\mathbb{R}$,
such that 
\begin{equation}
P\left(s'|s,a\right)=\sum_{k=1}^{K}\phi_{k}\left(s,a\right)\psi_{k}\left(s'\right)\label{eq:defn-linear-model}
\end{equation}
for every $(s,a)\in\mathcal{S}\times\mathcal{A}$ and $s'\in\mathcal{S}$.
\end{definition}

Readers familiar with linear MDP literatures might immediately recognize
that the above definition is the same as the structure imposed on
the probability transition kernel $P$ in the linear MDP model \cite{yang2019sample,jin2020provably,zanette2020frequentist,he2020logarithmic,touati2020efficient,wang2020reward,wei2021learning}.
However unlike linear MDP which also requires the reward function
$r(s,a)$ to be linear in the feature mapping $\phi(s,a)$, here we
do not impose any structural assumption on the reward.

\begin{example}[Tabular MDP]\label{ex:1}Each tabular MDP can be
viewed as a linear transition model with feature mapping $\phi(s,a)=\bm{e}_{(s,a)}\in\mathbb{R}^{\vert\mathcal{S}\vert\times\vert\mathcal{A}\vert}$
(i.e.~the vector with all entries equal to $0$ but the one corresponding
to $(s,a)$ equals to $1$) for all $(s,a)\in\mathcal{S}\times\mathcal{A}$.
To see this, we can check that Definition \ref{def:linear-model}
is satisfied with $K=|\mathcal{S}|\times|\mathcal{A}|$ and $\psi_{(s,a)}(s')=\mathbb{P}(s'\vert s,a)$
for each $s,s'\in\mathcal{S}$ and $a\in\mathcal{A}$. This example
is a sanity check of Definition \ref{def:linear-model}, which also
shows that our results (Theorem \ref{thm:model-based} and \ref{thm:syn-q-upper})
can recover previous results on tabular MDP \cite{agarwal2020model,li2021qlearning}
by taking $K=|\mathcal{S}|\times|\mathcal{A}|$. 

\end{example}

\begin{example}[Simplex Feature Space]\label{ex:2}If all feature
vectors $\{\phi(s,a)\}_{(s,a)\in\mathcal{S}\times\mathcal{A}}$ fall
in the probability simplex $\Delta_{K-1}$, a linear transition model
can be constructed by taking $\psi_{k}(\cdot)$ to be any probability
measure over $\mathcal{S}$ for all $k\in[K]$. 

\end{example}

A key observation is that the model size of linear transition model
with known feature mapping $\phi$ is $|\mathcal{S}|K$ (the number of coefficients $ \psi_{k}\left(s'\right)$ in \eqref{eq:defn-linear-model}), which is
still large when the state space $\mathcal{S}$ is large. In contrast,
it will be established later that to learn a near-optimal policy or
Q-function, we only need a much smaller number of samples, which depends
linearly on $K$ and is independent of $\vert\mathcal{S}\vert$.

Next, we introduce a critical assumption employed in prior literature
\cite{yang2019sample,zanette2019limiting,shariff2020efficient}.

\begin{assumption}[Anchor state-action pairs]\label{assumption:anchor}
Assume there exists a set of anchor state-action pairs $\mathcal{K\subset\mathcal{S}\times\mathcal{A}}$
with $|\mathcal{K}|=K$\footnote{Without loss of generality, one can always assume that the number of anchor state-action pairs equals to the feature dimension $K$. Interested readers are referred to Appendix \ref{appendix:dimension-number} for detailed argument.} such that for any $(s,a)\in\mathcal{S}\times\mathcal{A}$,
its corresponding feature vector can be expressed as a convex combination
of the feature vectors of anchor state-action pairs $\{(s,a):(s,a)\in\mathcal{K}\}$:
\begin{equation}
\phi\left(s,a\right)=\sum_{i:(s_{i},a_{i})\in\mathcal{K}}\lambda_{i}\left(s,a\right)\phi\left(s_{i},a_{i}\right)\quad\text{for}\quad\sum_{i=1}^{K}\lambda_{i}\left(s,a\right)=1\quad\text{and}\quad\lambda_{i}\left(s,a\right)\geq0.\label{eq:convex-coeff}
\end{equation}
Further, we assume that the vectors in $\{\phi(s,a):(s,a)\in\mathcal{K}\}$ are linearly independent.

\end{assumption}

We pause to develop some intuition of this assumption using Examples \ref{ex:1} and \ref{ex:2}. In Example \ref{ex:1}, it is straightforward to check that tabular MDPs satisfies
Assumption \ref{assumption:anchor} with $\mathcal{K}=\mathcal{S}\times\mathcal{A}$.
In terms of Example \ref{ex:2}, without loss of generality we can
assume that the subspace spanned by the features has full rank, i.e.~$\mathsf{span}\{\phi(s,a):(s,a)\in\mathcal{S}\times\mathcal{A}\}=\mathbb{R}^K$ (otherwise we can reduce the dimension of feature
space). Then we can also check that Example \ref{ex:2} satisfies
Assumption \ref{assumption:anchor} with arbitrary $\mathcal{K}\subseteq\mathcal{S}\times\mathcal{A}$
such that the vectors in $\{\phi(s,a):(s,a)\in\mathcal{K}\}$ are
linearly independent. In fact, this sort of ``anchor'' notion appears
widely in the literature: \cite{arora2012learning}
considers ``anchor word'' in topic modeling; \cite{donoho2004does}
defines ``separability'' in their study of non-negative matrix factorization; \cite{singh1995reinforcement} introduces ``aggregate'' in reinforcement
learning; \cite{duan2018state} studies ``anchor state'' in soft
state aggregation models. These concepts all bear some kind of resemblance
to our definition of anchor state-action pairs here. 

Throughout this paper, we assume that the feature mapping $\phi$ is known, which is a widely adopted assumption in previous literature \cite{yang2019sample,jin2020provably,zhou2020provably,he2020logarithmic,touati2020efficient,wang2020reward,wei2021learning}. In practice, large scale RL usually makes use of representation learning to obtain the
feature mapping $\phi$. Furthermore, the learned representations
can be selected to satisfy the anchor state-action pairs assumption
by design.

A useful implication of Assumption \ref{assumption:anchor} is that we can represent the transition kernel as
\begin{equation}
P\left(\cdot|s,a\right)=\sum_{i:(s_{i},a_{i})\in\mathcal{K}}\lambda_{i}\left(s,a\right)P\left(\cdot|s_{i},a_{i}\right),\label{eq:p-pk}
\end{equation}
This follows simply from substituting \eqref{eq:convex-coeff} into \eqref{eq:defn-linear-model} (see (\ref{eq:claim-proof}) in Appendix~\ref{sec:Notations} for a formal proof).

\section{Model-based RL with a generative model\label{sec:Model-based-RL}}

We start with studying model-based RL with a generative model in this section. We propose a model-based planning algorithm and show that it returns an $\varepsilon$-optimal policy with minimax optimal sample size. 

\subsection{Main results}

\paragraph{A generative model and an empirical MDP.}

We assume access to a generative model that provides us with independent samples from $M$. For each anchor state-action pair
$(s_{i},a_{i})\in\mathcal{K}$, we collect $N$ independent samples
$s_{i}^{(j)}\sim P(\cdot|s_{i},a_{i})$, $j\in[N]$. This allows us
to construct an empirical transition kernel $\widehat{P}$ where 
\begin{equation}
\widehat{P}\left(s'\,|\,s,a\right)=\sum_{i=1}^{K}\lambda_{i}\left(s,a\right)\cdot\left(\frac{1}{N}\sum_{j=1}^{N}\ind\left\{ s_{i}^{\left(j\right)}=s'\right\} \right),\label{eq:empirical-transition}
\end{equation}
for each $(s,a)\in\mathcal{S}\times\mathcal{A}$. Here, $\frac{1}{N}\sum_{j=1}^{N}\ind\{s_{i}^{(j)}=s'\}$
is an empirical estimate of $P(s'|s_{i},a_{i})$ and then (\ref{eq:p-pk})
is employed. With $\widehat{P}$ in hand, we can construct an empirical
MDP $\widehat{M}=(\mathcal{S},\mathcal{A},\widehat{P},r,\gamma)$.
Our goal here is to derive the sample complexity which guarantees that the optimal
policy of $\widehat{M}$ is an $\varepsilon$-optimal policy for the
true MDP $M$. The algorithm is summarized below. 

\begin{algorithm}[h]
\caption{Model-based RL with a generative model}

\label{alg:model-based-rl}\begin{algorithmic}

\STATE \textbf{{Inputs}}: a set of anchor state-action pairs $\mathcal{K}=\{(s_{i},a_{i}):i\in[K]\}$,
feature mapping $\phi:\mathcal{S}\times\mathcal{A}\to\mathbb{R}^{K}$,
any planning algorithm $\mathcal{P}$, target algorithmic error level
$\varepsilon_{\mathsf{opt}}$.

\STATE \textbf{For }$i=1,\cdots,K$ \textbf{do}

\STATE$\qquad$Draw $N$ independent samples $s_{i}^{(j)}\sim P(\cdot|s_{i},a_{i})$,
$j=1$, $\cdots$, $N$.

\STATE \textbf{End for}

\STATE Construct an empirical MDP $\widehat{M}=(\mathcal{S},\mathcal{A},\widehat{P},r,\gamma)$
where $\widehat{P}$ can be computed by (\ref{eq:empirical-transition}).

\STATE \textbf{{Output}} $\widehat{\pi}$ as an $\varepsilon_{\mathsf{opt}}$-optimal
policy of $\widehat{M}$ computed by the planning algorithm $\mathcal{P}$.

\end{algorithmic}
\end{algorithm}

Careful readers may note that in Algorithm
\ref{alg:model-based-rl}, $\{\lambda(s,a):(s,a)\in\mathcal{S}\times\mathcal{A}\}$ is used
in the construction of $\widehat{P}$, while $\{\lambda(s,a):(s,a)\in\mathcal{S}\times\mathcal{A}\}$
is not input into the algorithm. This is because given $\mathcal{K}$
and $\phi$ are known, $\{\lambda(s,a):(s,a)\in\mathcal{S}\times\mathcal{A}\}$
can be calculated explicitly. The following theorem provides theoretical guarantees for the
output policy $\widehat{\pi}$ of the chosen optimization algorithm
on the empirical MDP $\widehat{M}$.

\begin{theorem}\label{thm:model-based}Suppose that $\delta>0$ and $\varepsilon\in(0,(1-\gamma)^{-1/2}]$.
Let $\widehat{\pi}$ be the policy returned by Algorithm~\ref{alg:model-based-rl}.
Assume that
\begin{equation}
N\geq\frac{C\log\left(K/\left(\left(1-\gamma\right)\delta\right)\right)}{\left(1-\gamma\right)^{3}\varepsilon^{2}}\label{eq:model-sample-thm}
\end{equation}
for some sufficiently large constant $C>0$. Then with probability
exceeding $1-\delta$,
\begin{equation}
Q^{\star}(s,a)-Q^{\widehat{\pi}}(s,a)\leq\varepsilon+\frac{4\varepsilon_{\mathsf{opt}}}{1-\gamma},\label{eq:model-error}
\end{equation}
for every $(s,a)\in\mathcal{S}\times\mathcal{A}$. Here $\varepsilon_{\mathsf{opt}}$ is the target algorithmic error level in Algorithm \ref{alg:model-based-rl}.
\end{theorem}

We first remark that the two terms on the right hand side of (\ref{eq:model-error})
can be viewed as statistical error and algorithmic error, respectively.
The first term $\varepsilon$ denotes the statistical error coming
from the deviation of the empirical MDP $\widehat{M}$ from the true
MDP $M$. As the sample size $N$ grows, $\varepsilon$ could decrease
towards $0$. The other term $4\varepsilon_{\mathsf{opt}}/(1-\gamma)$
represents the algorithmic error where $\varepsilon_{\mathsf{opt}}$
is the target accuracy level of the planning algorithm applied to
$\widehat{M}$. Note that $\mathcal{\varepsilon_{\mathsf{opt}}}$
can be arbitrarily small if we run the planning algorithm (e.g.~value
iteration) for enough iterations. A few implications of this theorem
are in order.
\begin{itemize}[leftmargin=*]
\item \textit{Minimax-optimal sample complexity}. Assume that $\varepsilon_{\mathsf{opt}}$
is made negligibly small, e.g.~$\varepsilon_{\mathsf{opt}}=O((1-\gamma)\varepsilon)$
to be discussed in the next point. Note that we draw $N$ independent
samples for each state-action pair $(s,a)\in\mathcal{K}$, therefore
the requirement \eqref{eq:model-sample-thm}
for finding an $O(\varepsilon)$-optimal policy translates into the
following sample complexity requirement
\[
\widetilde{O}\left(\frac{K}{\left(1-\gamma\right)^{3}\varepsilon^{2}}\right).
\]
This matches the minimax optimal lower bound (up to a logarithm factor)
established in \cite[Theorem 1]{yang2019sample} for feature-based
MDP. In comparison, for tabular MDP the minimax optimal sample complexity
is $\widetilde{\Omega}((1-\gamma)^{-3}\varepsilon^{-2}\vert\mathcal{S}\vert\vert\mathcal{A}\vert)$
\cite{azar2013minimax, agarwal2020model}. Our sample complexity scales linearly with
$K$ instead of $\vert\mathcal{S}\vert\vert\mathcal{A}\vert$ for
tabular MDP as desired.
\item \textit{Computational complexity}. An advantage of Theorem \ref{thm:model-based}
is that it incorporates the use of any efficient planning algorithm
applied to the empirical MDP $\widehat{M}$. Classical algorithms
include Q-value iteration (QVI) or policy iteration (PI) \cite{puterman2014markov}.
For example, QVI achieves the target level $\varepsilon_{\mathsf{opt}}$
in $O((1-\gamma)^{-1}\log\varepsilon_{\text{opt}}^{-1})$ iterations,
and each iteration takes time proportional to $O(NK+|\mathcal{S}||\mathcal{A}|K)$.
To learn an $O(\varepsilon)$-optimal policy, which requires sample complexity (\ref{eq:model-sample-thm})
and the target level $\varepsilon_{\mathsf{opt}}=O((1-\gamma)\varepsilon)$,
the overall running time is
\[
\widetilde{O}\left(\frac{\left|\mathcal{S}\right|\left|\mathcal{A}\right|K}{1-\gamma}+\frac{K}{\left(1-\gamma\right)^{4}\varepsilon^{2}}\right).
\]
In comparison, for the tabular MDP the corresponding running time
is $\widetilde{O}((1-\gamma)^{-4}\varepsilon^{-2}\vert\mathcal{S}\vert\vert\mathcal{A}\vert)$
\cite{agarwal2020model}. This suggests that under the feature-based
linear transition model, the computational complexity is $\min\{\vert\mathcal{S}\vert\vert\mathcal{A}\vert/K,(1-\gamma)^{-3}\varepsilon^{-2}/K\}$
times lower than that for the tabular MDP (up to logarithm factors),
which is significantly more efficient when $K$ is not too large.

\item \textit{Stability }\emph{vis-\`a-vis}\textit{ model misspecification}. A more
general version of Theorem~\ref{thm:model-based} (Theorem \ref{thm:model-based-general}
in Appendix \ref{subsec:Proof-of-Theorem-model}) shows that when
$P$ approximately (instead of exactly) admits the linear transition
model, we can still achieve some meaningful result. Specifically,
if there exists a linear transition kernel $\widetilde{P}$ obeying
$\max_{(s,a)\in\mathcal{S}\times\mathcal{A}}\Vert\widetilde{P}(\cdot|s,a)-P(\cdot|s,a)\Vert_{1}\leq\xi$
for some $\xi\geq0$, we can show that $\widehat{\pi}$ returned by
Algorithm \ref{alg:model-based-rl} (with slight modification) satisfies
\[
Q^{\star}(s,a)-Q^{\widehat{\pi}}(s,a)\leq\varepsilon+\frac{4\varepsilon_{\mathsf{opt}}}{1-\gamma}+\frac{22\xi}{(1-\gamma)^2},
\]
for every $(s,a)\in\mathcal{S}\times\mathcal{A}$. 
This shows that the model-based method is stable vis-a\'-vis model misspecification.
Interested readers are referred to Appendix \ref{subsec:Proof-of-Theorem-model}
for more details.
\end{itemize}
In Algorithm \ref{alg:model-based-rl}, the reward function $r$
is assumed to be known. If the information of $r$ is unavailable, an alternative is to assume that $r$ is linear with respect
to the feature mapping $\phi$, i.e. $r(s,a)=\theta^{\top}\phi(s,a)$
for every $(s,a)\in\mathcal{S}\times\mathcal{A}$, which is widely
adopted in linear MDP literature \cite{he2020logarithmic,jin2020provably,wang2020reward,wei2021learning}. Under this linear
assumption, one can obtain $\theta$ by solving the following linear
system of equations 
\begin{equation}
r\left(s,a\right)=\theta^{\top}\phi\left(s,a\right),\quad\forall\left(s,a\right)\in\mathcal{K},
\end{equation}
which can be constructed by the observed reward $r(s,a)$ for all
anchor state-action pairs.

After the publication of this paper, we noticed that \cite{cui2020plug} achieved similar result as our Theorem \ref{thm:model-based}. 

\section{Model-free RL---vanilla Q Learning\label{sec:Model-free-RL}}
In this section, we turn to study one of the most popular model-free RL algorithms---Q-learning. We provide tight sample complexity bound for vanilla Q-learning under the feature-based linear transition model, which shows its sample-efficiency (depends on $\vert K \vert$ instead of $\vert\mathcal{S}\vert$ or $\vert\mathcal{A}\vert$) and sub-optimality in the dependency on the effective horizon.

\subsection{Q-learning algorithm}

The vanilla Q-learning algorithm maintains a Q-function estimate
$Q_{t}:\mathcal{S}\times\mathcal{A}\rightarrow\mathbb{R}$ for all
$t\ge0$, with initialization $Q_0$ obeying $0\leq Q_{0}(s,a)\leq\frac{1}{1-\gamma}$
for every $(s,a)\in\mathcal{S}\times\mathcal{A}$. Assume we have access to a generative model.
In each iteration $t\geq1$, we collect an independent sample $s_{t}(s,a)\sim P(\cdot|s,a)$
for every anchor state-action pair $(s,a)\in\mathcal{K}$ and define $Q_{\mathcal{K}}^{\left(t\right)}:\mathcal{K}\rightarrow\mathbb{R}$
to be 
\[
Q_{\mathcal{K}}^{\left(t\right)}\left(s,a\right)\coloneqq\max_{a'\in\mathcal{A}}Q_t\left(s_{t},a'\right),\qquad s_{t}\equiv s_{t}\left(s,a\right)\sim P\left(\cdot|s,a\right).
\]
Then given the learning rate $\eta_t\in(0,1]$, the algorithm adopts the following update rule to update all entries of the Q-function estimate
\[
Q_{t}=\left(1-\eta_{t}\right)Q_{t-1}+\eta_{t}\mathcal{T}_{\mathcal{K}}^{\left(t\right)}\left(Q_{t-1}\right).
\]
Here, $\mathcal{T}_{\mathcal{K}}^{\left(t\right)}$ is an empirical Bellman operator associated with the linear
transition model $M$ and the set $\mathcal{K}$ and is given by 
\[
\mathcal{T}_{\mathcal{K}}^{\left(t\right)}\left(Q\right)\left(s,a\right)\coloneqq r\left(s,a\right)+\gamma\lambda\left(s,a\right)Q_{\mathcal{K}}^{\left(t\right)},
\]
where \eqref{eq:p-pk} is used in the construction.
Clearly,  this newly defined operator $\mathcal{T}_{\mathcal{K}}^{(t)}$
is an unbiased estimate of the famous Bellman operator $\mathcal{T}$ \cite{bellman1952theory} defined as
\[
\forall\left(s,a\right)\in\mathcal{S}\times\mathcal{A}:\qquad\mathcal{T}\left(Q\right)\left(s,a\right)\coloneqq r\left(s,a\right)+\gamma\mathbb{E}_{s'\sim P\left(\cdot|s,a\right)}\left[\max_{a'\in\mathcal{A}}Q\left(s',a'\right)\right].
\]
A critical property is that the Bellman operator $\mathcal{T}$ is contractive with a unique fixed point which is the optimal Q-function $Q^{\star}$\cite{bellman1952theory}. To solve the fixed-point equation $\mathcal{T}(Q^{\star})=Q^{\star}$, Q-learning was then introduced by \cite{watkins1992q} based on the idea of stochastic approximation
\cite{robbins1951stochastic}. This procedure is precisely described  in Algorithm~\ref{alg:syn-q}.


\begin{algorithm}[h]
\caption{Vanilla Q-learning for infinite-horizon discounted MDPs}

\label{alg:syn-q}\begin{algorithmic}

\STATE \textbf{{inputs}}: learning rates $\left\{ \eta_{t}\right\} $,
number of iterations $T$, discount factor $\gamma$, initial estimate
$Q_{0}$.

\STATE \textbf{for }$t=1,\ldots,T$ \textbf{do}

\STATE$\qquad$Draw $s_{t}\left(s,a\right)\sim P\left(\cdot|s,a\right)$
for each $\left(s,a\right)\in\mathcal{K}$.

\STATE$\qquad$Compute $\bm{Q}_{t}$ according to the update rule\vspace{-1em}

\[
Q_{t}=\left(1-\eta_{t}\right)Q_{t-1}+\eta_{t}\mathcal{T}_{\mathcal{K}}^{\left(t\right)}\left(Q_{t-1}\right).
\]

\STATE \textbf{end for}

\end{algorithmic}
\end{algorithm}

\subsection{Main results}

We are now ready to provide our main result for vanilla Q-learning, assuming sampling access to a generative model.
\begin{theorem}\label{thm:syn-q-upper} Consider any $\delta\in(0,1)$
and $\varepsilon\in(0,1]$. Assume that for any $0\leq t\leq T$,
the learning rates satisfy 
\begin{equation}
\frac{1}{1+\frac{c_{1}\left(1-\gamma\right)T}{\log^{2}T}}\leq\eta_{t}\leq\frac{1}{1+\frac{c_{2}\left(1-\gamma\right)t}{\log^{2}T}}\label{eq:syn-q-stepsize}
\end{equation}
for some sufficiently small universal constants $c_{1}\geq c_{2}>0$.
Suppose that the total number of iterations $T$ exceeds 
\begin{equation}
T\geq\frac{C_{3}\log\left(KT/\delta\right)\log^{4}T}{\left(1-\gamma\right)^{4}\varepsilon^{2}}\label{eq:syn-q-iteration-complexity}
\end{equation}
for some sufficiently large universal constant $C_{3}>0$. If the
initialization obeys $0\leq Q_{0}(s,a)\leq\frac{1}{1-\gamma}$ for
any $(s,a)\in\mathcal{S}\times\mathcal{A}$, then with probability
exceeding $1-\delta$, the output $Q_{T}$ of Algorithm \ref{alg:syn-q}
satisfies
\begin{equation}
\max_{\left(s,a\right)\in\mathcal{S}\times\mathcal{A}}\left|Q_{T}\left(s,a\right)-Q^{\star}\left(s,a\right)\right|\leq\varepsilon.\label{eq:model-free-Q-T}
\end{equation}
In addition, let $\pi_{T}$ (resp.~$V_{T}$)
to be the policy (resp.~value function) induced by $Q_{T}$, then
one has
\begin{equation}
	\max_{s\in\mathcal{S}}\left|V^{\pi_{T}}\left(s\right)-V^{\star}\left(s\right)\right|\leq\frac{2\gamma\varepsilon}{1-\gamma}.\label{eq:model-q-V}
\end{equation}
\end{theorem}

This theorem provides theoretical guarantees on the performance of
Algorithm \ref{alg:syn-q}. A few implications of this theorem are
in order.
\begin{itemize}[leftmargin=*]
\item \textit{Learning rate}. The condition (\ref{eq:syn-q-stepsize})
accommodates two commonly adopted choice of learning rates: (i) linearly
rescaled learning rates $\eta_{t}=[1+c_{2}(1-\gamma)t/\log^{2}T]^{-1}$,
and (ii) iteration-invariant learning rates $\eta_{t}\equiv[1+c_{1}(1-\gamma)T/\log^{2}T]$.
Interested readers are referred to the discussions in \cite[Section 3.1]{li2021qlearning}
for more details on these two learning rate schemes.
\item \textit{Tight sample complexity bound}. Note that we draw $K$ independent
samples in each iteration, therefore the iteration complexity (\ref{eq:syn-q-iteration-complexity})
can be translated into the sample complexity bound $TK$ in order for Q-learning
to achieve $\varepsilon$-accuracy:
\begin{equation}
\widetilde{O}\left(\frac{K}{\left(1-\gamma\right)^{4}\varepsilon^{2}}\right).\label{eq:Q-sample}
\end{equation}
As we will see shortly, this result improves the state-of-the-art
sample complexity bound presented in \cite[Theorem 2]{yang2019sample}
. In addition, the dependency on the effective horizon $(1-\gamma)^{-4}$
matches the lower bound established in \cite[Theorem 2]{li2021qlearning}
for vanilla Q-learning using either learning rate scheme covered in
the previous remark, suggesting that our sample complexity bound (\ref{eq:Q-sample})
is sharp. 
\item \textit{Stability }\emph{vis-\`a-vis}\textit{ model misspecification}. Just
like the model-based approach, we can also show that Q-learning is
also stable vis-a\'-vis model misspecification when $P$ approximately admits
the linear transition model. We refer interested readers to Theorem
\ref{thm:syn-q-upper-general} in Appendix \ref{subsec:Proof-of-Theorem-model}
for more details.
\end{itemize}

\paragraph{Comparison with \cite{yang2019sample}.} We compare our result with the sample complexity bounds for Q-learning under the feature-based linear transition model in \cite{yang2019sample}.
\begin{itemize}[leftmargin=*]
\item We first compare our result with \cite[Theorem 2]{yang2019sample},
which is, to the best of our knowledge, the state-of-the-art theory
for this problem. When there is no model misspecification, \cite[Theorem 2]{yang2019sample}
showed that in order for their Phased Parametric Q-learning\footnote{The
	difference between Algorithm \ref{alg:syn-q} and Phased Parametric Q-Learning in
	\cite{yang2019sample} is that Algorithm \ref{alg:syn-q} maintains and updates a $Q$-function
	estimate $Q_{t}$, while Phased Parametric Q-Learning parameterized $Q$-function by 
	\begin{equation*}
	Q_{w}\left(s,a\right):= r\left(s,a\right)+\gamma\phi\left(s,a\right)^{\top}w,
	\end{equation*}
	and then updates the parameters $w$.} (Algorithm
1 therein) to learn an $\varepsilon$-optimal policy, the sample size
needs to be
\[
\widetilde{O}\left(\frac{K}{\left(1-\gamma\right)^{7}\varepsilon^{2}}\right).
\]
Note that (\ref{eq:Q-sample}) is the sample complexity required for
entrywise $\varepsilon$-accurate estimate of the optimal Q-function,
thus a fair comparison requires to use the sample complexity for learning an $\varepsilon$-optimal policy deduced from \eqref{eq:model-q-V}, which is
\[
\widetilde{O}\left(\frac{K}{\left(1-\gamma\right)^{6}\varepsilon^{2}}\right).
\]
Hence, our sample complexity
improves upon previous work by a factor at least on the order of $(1-\gamma)^{-1}$. However it is worth mentioning that \cite[Theorem 2]{yang2019sample} is built upon weaker conditions $\sum_{i=1}^{K}\lambda_{i}(s,a)=1$
and $\sum_{i=1}^{K}\vert\lambda_{i}(s,a)\vert\leq L$ for some $L\geq1$,
which does not require $\lambda_{i}(s,a)\geq0$. Our result holds
under Assumption \ref{assumption:anchor}, which requires $\sum_{i=1}^{K}\lambda_{i}(s,a)=1$
and $\lambda_{i}(s,a)\geq0$. Under the current analysis framework,
it is difficult to obtain tight sample complexity bounds without assuming
$\lambda_{i}(s,a)\geq0$.
\item Besides vanilla Q-learning, \cite{yang2019sample} also proposed a
new variant of Q-learning called Optimal Phased Parametric Q-Learning
(Algorithm 2 therein), which is essentially Q-learning with variance
reduction. \cite[Theorem 3]{yang2019sample} showed that the sample
complexity for this algorithm is
\[
\widetilde{O}\left(\frac{K}{(1-\gamma)^{3}\varepsilon^{2}}\right),
\]
which matches minimax optimal lower bound (up to a logarithm factor)
established in \cite[Theorem 1]{yang2019sample}. Careful reader might
notice that this sample complexity bound is better than ours for vanilla
Q-learning. We emphasize that as elucidated in the second implication
under Theorem \ref{thm:syn-q-upper}, our result is already tight
for vanilla Q-learning. This observation reveals that while the sample
complexity for vanilla Q-learning is provably sub-optimal, the variants
of Q-learning can have better performance and achieve minimax optimal
sample complexity. 
\end{itemize}
We conclude this section by comparing model-based and model-free approaches.
Theorem \ref{thm:model-based} shows that the sample complexity of
the model-based approach is minimax optimal, whilst vanilla Q-learning, perhaps 
the most commonly adopted model-free method, is sub-optimal according
to Theorem \ref{thm:syn-q-upper}. However this does not mean that model-based
method is better than model-free ones since (i) some variants of Q-learning
(see \cite[Algorithm 2]{yang2019sample} for example) also has minimax
optimal sample complexity; and (ii) in many applications it might
be unrealistic to estimate the model in advance. 

\section{A glimpse of our technical approaches} \label{sec:glimpse}

The establishment of Theorems \ref{thm:model-based} and \ref{thm:syn-q-upper}
calls for a series of technical novelties in the proof. In what follows,
we briefly highlight our key technical ideas and novelties.
\begin{itemize}[leftmargin=*]
\item For the model-based approach, we employ ``leave-one-out'' analysis
to decouple the complicated statistical dependency between the empirical
probability transition model $\widehat{P}$ and the corresponding
optimal policy. Specifically, \cite{agarwal2020model} proposed to construct
a collection of auxiliary MDPs where each one of them
leaves out a single state $s$ by setting $s$ to be an absorbing state and keeping everything else untouched. We tailor this high level idea to the needs of linear transition model, then the independence
between the newly constructed MDP with absorbing state $s$ and data
samples collected at state $s$ will facilitate our analysis, as
detailed in Lemma \ref{lemma:loo-model-based}. This ``leave-one-out''
type of analysis has been utilized in studying numerous problems by a long line
of work, such as \cite{el2018impact,ma2018implicit,wainwright2019stochastic,chen2019spectral,chen2020noisy,chen2020bridging,chen2020convex},
just to name a few. 
\item To obtain tighter sample complexity bound than the previous one $\widetilde{O}(\frac{K}{(1-\gamma)^{7}\varepsilon^{2}})$
in \cite{yang2019sample} for vanilla Q-learning, we invoke Freedman's
inequality \cite{freedman1975tail} for the concentration of an error
term with martingale structure as illustrated in Appendix \ref{subsec:Proof-of-Theorem-syn-q},
while the classical ones used in analyzing Q-learning are Hoeffding's inequality and Bernstein's
inequality \cite{yang2019sample}. The use of Freedman's inequality
helps us establish a recursive relation on $\{\Vert Q_{t}-Q^{\star}\Vert_{\infty}\}_{t=0}^{T}$,
which consequently leads to the performance guarantee (\ref{eq:model-free-Q-T}). It is worth
mentioning that \cite{li2021qlearning} also studied vanilla Q-learning in the tabular
MDP setting and adopted Freedman's inequality, while we emphasize
that it requires a lot of efforts and more delicate analyses in order
to study linear transition model and also allow for model misspecification
in the current paper, as detailed in the appendices.
\end{itemize}

\section{Additional related literature}

To remedy the issue of prohibitively high sample complexity, there exists
a substantial body of literature proposing and studying many structural
assumptions and complexity notions under different settings. This
current paper focuses on linear transition model which is studied
in MDP by numerous previous works \cite{yang2019sample,jin2020provably,yang2020reinforcement,zhou2020provably,modi2020sample,hao2020sparse,wang2020reward,touati2020efficient,he2020logarithmic,wei2021learning, cui2020plug}. Among them, \cite{yang2019sample} studied linear transition model and provided tight sample complexity bounds for a new variant of Q-learning with the help of variance reduction. \cite{jin2020provably} focused on linear MDP and designed an algorithm called ``Least-Squares Value Iteration with UCB'' with both polynomial runtime and polynomial sample complexity without accessing generative model. \cite{wang2020reward} extended the study of linear MDP to the framework of reward-free reinforcement learning. \cite{zhou2020provably} considered a different feature mapping called linear kernel MDP and devised an algorithm with polynomial regret bound without generative model. Other popular structure assumptions include: \cite{wen2017efficient}
studied fully deterministic transition dynamics; \cite{jiang2017contextual}
introduced Bellman rank and proposed an algorithm which needs sample
size polynomially dependent on Bellman rank to obtain a near-optimal policy
in contextual decision processes; \cite{du2019provably} assumed that the
value function has low variance compared to the mean for all deterministic
policy;  \cite{melo2007q,parr2008analysis,azizzadenesheli2018efficient,zanette2020learning}
used linear model to approximate the value function; \cite{li2021sample} assumed that the optimal Q-function can be linearly-parameterized by the features.

Apart from the linear transition model, another notion adopted in this work is the generative model, whose role in discounted
MDP has been studied by extensive literature. The concept of generative model was originally
introduced by \cite{kearns1999finite}, and then widely adopted in
numerous works, including \cite{kakade2003sample,azar2013minimax,yang2019sample,wainwright2019variance,agarwal2020model,pananjady2020instance},
to name a few. Specifically, it is assumed that a generative model
of the studied MDP is available and can be queried for every state-action
pair and output the next state. Among previous works, \cite{azar2013minimax}
proved that the minimax lower bound on the sample complexity to obtain
an $\varepsilon$-optimal policy was $\widetilde{\Omega}(\frac{|\mathcal{S}||\mathcal{A}|}{(1-\gamma)^{3}\varepsilon^{2}})$.
\cite{azar2013minimax} also showed that model-based approach can
output an $\varepsilon$-optimal value function with near-optimal
sample complexity for $\varepsilon\in(0,1)$. Then \cite{agarwal2020model}
made significant progress on the challenging problem of establishing
minimax optimal sample complexity in estimating an $\varepsilon$-optimal
policy with the help of ``leave-one-out'' analysis. 

In addition, after being proposed in \cite{watkins1989learning},
Q-learning has become the focus of a rich line of research \cite{watkins1992q,bertsekas1996neuro,kearns1999finite,even2003learning,azar2011speedy,jin2018q,wainwright2019stochastic,chen2019performance,li2020sample,xu2020finite}.
Among them, \cite{chen2019performance,li2020sample,xu2020finite}
studied Q-learning in the presence of Markovian data, i.e.~a single
sample trajectory. In contrast, under the generative setting of Q-learning
where a fresh sample can be drawn from the simulator at each iteration,
\cite{wainwright2019variance} analyzed a variant of Q-learning with the help of variance reduction, which was proved to enjoy minimax optimal sample complexity $\widetilde{O}(\frac{|\mathcal{S}||\mathcal{A}|}{(1-\gamma)^{3}\varepsilon^{2}})$. Then more recently,
\cite{li2021qlearning} improved the lower bound of the vanilla Q-learning algorithm in terms of its scaling with $\frac{1}{1-\gamma}$ and proved a matching upper bound $\widetilde{O}(\frac{|\mathcal{S}||\mathcal{A}|}{(1-\gamma)^{4}\varepsilon^{2}})$. 

\section{Discussion\label{sec:Discussion}}

This paper studies sample complexity of both model-based and model-free
RL under a discounted infinite-horizon MDP with feature-based linear
transition model. We establish tight sample complexity bounds for
both model-based approaches and Q-learning, which scale linearly with
the feature dimension $K$ instead of $|\mathcal{S}|\times|\mathcal{A}|$,
thus considerably reduce the required sample size for large-scale
MDPs when $K$ is relatively small. Our results are sharp, and the
sample complexity bound for the model-based approach matches the minimax
lower bound. The current work suggests a couple of directions for
future investigation, as discussed in detail below.
\begin{itemize}[leftmargin=*]
\item \emph{Extension to episodic MDPs.} An interesting direction for future
research is to study linear transition model in episodic MDP. This
focus of this work is infinite-horizon discounted MDPs, and hopefully
the analysis here can be extended to study the episodic MDP as well
(\cite{dann2015sample,dann2017unifying,azar2017minimax,jiang2018open,wang2020long,he2020logarithmic}). 
\item \emph{Continuous state and action space. }The state and action spaces
in this current paper are still assumed to be finite, since the proof
relies heavily on the matrix operations. However, we expect that the
results can be extended to accommodate continuous state and action
space by employing more complicated analysis.

\item \emph{Accommodating entire range of $\varepsilon$. }Since both value
functions and Q-functions can take value in $[0,(1-\gamma)^{-1}]$,
ideally our theory should cover all choices of $\varepsilon\in(0,(1-\gamma)^{-1}]$.
However we require that $\varepsilon\in(0,(1-\gamma)^{-1/2}]$ in
Theorem \ref{thm:model-based} and $\varepsilon\in(0,1]$ in Theorem
\ref{thm:syn-q-upper}. While most of the prior works like \cite{agarwal2020model,yang2019sample}
also impose these restrictions, a recent work \cite{li2020breaking}
proposed a perturbed model-based planning algorithm and proved minimax optimal
guarantees for any $\varepsilon\in(0,(1-\gamma)^{-1}]$. While their
work only focused on model-based RL under tabular MDP, an interesting
future direction is to improve our theory to accommodate any $\varepsilon\in(0,(1-\gamma)^{-1}]$.
\item \emph{General function approximation. }Another future direction is
to extend the study to more general function approximation starting
from linear structure covered in this paper. There exists a rich body
of work proposing and studying different structures, such as linear
value function approximation \cite{melo2007q,parr2008analysis,azizzadenesheli2018efficient,zanette2020learning},
linear MDPs with infinite dimensional features \cite{agarwal2020pc},
Eluder dimension \cite{wang2020reinforcement}, Bellman rank \cite{jiang2017contextual}
and Witness rank \cite{sun2019model}, etc. Therefore, it is hopeful
to investigate these settings and improve the sample efficiency.
\end{itemize}

\section*{Acknowledgements}

B.~Wang is supported in part by Gordon Y.~S.~Wu Fellowships in Engineering. Y.~Yan is supported in part by ARO grant W911NF-20-1-0097 and NSF grant CCF-1907661. Part of this work was done while Y.~Yan was visiting the Simons Institute for the Theory of Computing. J.~Fan is supported in part by the ONR grant N00014-19-1-2120 and the NSF grants DMS-1662139, DMS-1712591, DMS-2052926, DMS-2053832, and the NIH grant 2R01-GM072611-15. 

\appendix

\section{Notations\label{sec:Notations}}

In this section we gather the notations that will be used throughout the appendix.

For any vectors $\bm{u}=[u_{i}]_{i=1}^n\in\mathbb{R}^{n}$ and
$\bm{v}=[u_{i}]_{i=1}^n\in\mathbb{R}^{n}$, let $\bm{u}\,\circ\,\bm{v}=[u_{i}v_{i}]_{i=1}^{n}$
denote the Hadamard product of $\bm{u}$ and $\bm{v}$. We slightly abuse notations to use $\sqrt{\cdot}$ and $|\cdot|$
to define entry-wise operation, i.e.~for any vector $\bm{v}=[v_i]_{i=1}^n$ denote $\sqrt{\bm{v}}\coloneqq[\sqrt{v_{i}}]_{i=1}^{n}$
and $|\bm{v}|\coloneqq[|v_{i}|]_{i=1}^{n}$. Furthermore, the binary
notations $\leq$ and $\geq$ are both defined in entry-wise manner,
i.e.~$\bm{u}\leq\bm{v}$ (resp.~$\bm{u}\geq\bm{v}$) means $u_{i}\leq v_{i}$
(resp.~$u_{i}\geq v_{i}$) for all $1\leq i\leq n$. For a collection
of vectors $\bm{v}_{1}$, $\cdots$, $\bm{v}_{m}\in\mathbb{R}^{n}$
with $\bm{v}_{i}=[v_{i,j}]_{j=1}^{n}\in\mathbb{R}^{n}$, we define
the max operator to be $\max_{1\leq i\leq m}\bm{v}_{i}\coloneqq[\max_{1\leq i\leq m}v_{i,j}]_{j=1}^{n}$. 

For any matrix $\bm{M}\in\mathbb{R}^{m\times n}$, $\Vert\bm{M}\Vert_{1}$
is defined as the largest row-wise $\ell_{1}$ norm of $\bm{M}$,
i.e.~$\Vert\bm{M}\Vert_{1}\coloneqq\max_{i}\sum_{j}|M_{i,j}|$. In
addition, we define $\bm{1}$ to be a vector with all the entries
being 1, and $\bm{I}$ be the identity matrix. To express the probability
transition function $P$ in matrix form, we define the matrix $\bm{P}\in\mathbb{R}^{\left|\mathcal{S}\right|\left|\mathcal{A}\right|\times\left|\mathcal{S}\right|}$
to be a matrix whose $(s,a)$-th row $\bm{P}_{s,a}$ corresponds to
$P(\cdot|s,a)$. In addition, we define $\bm{P}^{\pi}$ to be the
probability transition matrix induced by policy $\pi$, i.e.~$\bm{P}_{(s,a),(s',a')}^{\pi}=\bm{P}_{s,a}(s')\ind_{\pi(s')=a'}$
for all state-action pairs $(s,a)$ and $(s',a')$. We define $\pi_{t}$
to be the policy induced by $Q_{t}$, i.e.~$Q_{t}(s,\pi_{t}(s))=\max_{a}Q_{t}(s,a)$
for all $s\in\mathcal{S}$. Furthermore, we denote the reward function
$r$ by vector $\bm{r}\in\mathbb{R}^{\left|\mathcal{S}\right|\left|\mathcal{A}\right|}$,
i.e.~the $(s,a)$-th element of $\bm{r}$ equals $r(s,a)$. In the
same manner, we define $\bm{V}^{\pi}\in\mathbb{R}^{\left|\mathcal{S}\right|}$,
$\bm{V}^{\star}\in\mathbb{R}^{\left|\mathcal{S}\right|}$, $\bm{V}_{t}\in\mathbb{R}^{\left|\mathcal{S}\right|}$,
$\bm{Q}^{\pi}\in\mathbb{R}^{\left|\mathcal{S}\right|\left|\mathcal{A}\right|}$,
$\bm{Q}^{\star}\in\mathbb{R}^{\left|\mathcal{S}\right|\left|\mathcal{A}\right|}$
and $\bm{Q}_{t}\in\mathbb{R}^{\left|\mathcal{S}\right|\left|\mathcal{A}\right|}$
to represent $V^{\pi}$, $V^{\star}$, $V_{t}$, $Q^{\pi}$, $Q^{\star}$
and $Q_{t}$ respectively. By using these notations, we can rewrite
the Bellman equation as 
\begin{equation}
\bm{Q}^{\pi}=\bm{r}+\gamma\bm{P}\bm{V}^{\pi}=\bm{r}+\gamma\bm{P}^{\pi}\bm{Q}^{\pi}.\label{eq:equation-q}
\end{equation}
Further, for any vector $\bm{V}\in\mathbb{R}^{|\mathcal{S}|}$, let $\mathsf{Var}_{\bm{P}}(\bm{V})\in\mathbb{R}^{\left|\mathcal{S}\right|\left|\mathcal{A}\right|}$ be 
\begin{equation}
\mathsf{Var}_{\bm{P}}\left(\bm{V}\right)\coloneqq\bm{P}\left(\bm{V}\circ\bm{V}\right)-\left(\bm{P}\bm{V}\right)\circ\left(\bm{P}\bm{V}\right),\label{eq:defn-varpv}
\end{equation}
and define $\mathsf{Var}_{\bm{P}_{s,a}}(\bm{V})\in\mathbb{R}$ to
be 
\begin{equation}
\mathsf{Var}_{\bm{P}_{s,a}}\left(\bm{V}\right)\coloneqq\bm{P}_{s,a}\left(\bm{V}\circ\bm{V}\right)-\left(\bm{P}_{s,a}\bm{V}\right)^{2},\label{eq:defn-varpsav}
\end{equation}
where $\bm{P}_{s,a}$ is the $(s,a)$-th row of $\bm{P}$.

Next, we reconsider Assumption \ref{assumption:anchor}. For any state-action pair $(s,a)$, we define vector $\bm{\lambda}(s,a)\in\mathbb{R}^{K}$ (resp.~$\bm{\phi}(s,a)\in\mathbb{R}^{K}$)
with $\bm{\lambda}(s,a)=[\lambda_{i}(s,a)]_{i=1}^{K}$ (resp.~$\bm{\phi}(s,a)=[\phi_{i}(s,a)]_{i=1}^{K}$) and matrix
$\bm{\Lambda}\in\mathbb{R}^{\left|\mathcal{S}\right|\left|\mathcal{A}\right|\times K}$ (resp.~$\bm{\Phi}\in\mathbb{R}^{\left|\mathcal{S}\right|\left|\mathcal{A}\right|\times K}$)
whose $(s,a)$-th row corresponds to $\bm{\lambda}(s,a)^{\top}$ (resp.~$\bm{\phi}(s,a)^{\top}$). Define vector $\bm{\psi}(s,a)\in\mathbb{R}^{K}$
with $\bm{\psi}(s,a)=[\psi_{i}(s,a)]_{i=1}^{K}$ and matrix
$\bm{\Psi}\in\mathbb{R}^{K\times \left|\mathcal{S}\right|}$
whose $(s,a)$-th column corresponds to $\bm{\psi}(s,a)^{\top}$. Further, let $\bm{P}_\mathcal{K}\in\mathbb{R}^{K\times|\mathcal{S}|}$ (resp.~$\bm{\Phi}_\mathcal{K}\in\mathbb{R}^{K\times K}$) to be a submatrix of $\bm{P}$ (resp.~$\bm{\Phi}$) formed by concatenating the rows $\{\bm{P}_{s,a}, (s,a)\in\mathcal{K}\}$ (resp.~$\{\bm{\Phi}_{s,a}, (s,a)\in\mathcal{K}\}$). By using
the previous notations, we can express the relations in Definition \ref{def:linear-model} and Assumption \ref{assumption:anchor} as $\bm{P}_{\mathcal{K}}=\bm{\Phi}_\mathcal{K}\bm{\Psi}$, $\bm{P}=\bm{\Phi}\bm{\Psi}$ and $\bm{\Phi}=\bm{\Lambda}\bm{\Phi}_\mathcal{K}$. Note that Assumption \ref{assumption:anchor} suggests $\bm{\Phi}_\mathcal{K}$ is invertible. Taking these equations collectively yields
\begin{equation}
\bm{P}=\bm{\Phi}\bm{\Psi}=\bm{\Phi}\bm{\Phi}_{\mathcal{K}}^{-1}\bm{P}_{\mathcal{K}}=\bm{\Lambda}\bm{\Phi}_{\mathcal{K}}\bm{\Phi}_{\mathcal{K}}^{-1}\bm{P}_{\mathcal{K}}=\bm{\Lambda}\bm{P}_{\mathcal{K}},\label{eq:claim-proof}
\end{equation}
which is reminiscent of the anchor word condition in topic modelling
\cite{arora2012learning}. In addition, for each iteration $t$, we denote the collected samples as $\{s_t(s,a)\}_{(s,a)\in\mathcal{K}}$ and
define a matrix $\widehat{\bm{P}}_{\mathcal{K}}^{(t)}\in\{0,1\}^{K\times\left|\mathcal{S}\right|}$
to be
\begin{equation}
\widehat{\bm{P}}_{\mathcal{K}}^{(t)}\left(\left(s,a\right),s'\right)\coloneqq\begin{cases}
1, & \text{if }s'=s_{t}\left(s,a\right)\\
0, & \text{otherwise}
\end{cases}\label{eq:defn-PKt}
\end{equation}
for any $(s,a)\in\mathcal{K}$ and $s'\in\mathcal{S}$. Further, we
define $\widehat{\bm{P}}_{t}=\bm{\Lambda}\widehat{\bm{P}}_{\mathcal{K}}^{(t)}$.
Then it is obvious to see that $\widehat{\bm{P}}_{t}$ has nonnegative
entries and unit $\ell_{1}$ norm for each row due to Assumption \ref{assumption:anchor},
i.e.~$\Vert\widehat{\bm{P}}_{t}\Vert_{1}=1$.

\section{Analysis of model-based RL (Proof of Theorem \ref{thm:model-based}) \label{subsec:Proof-of-Theorem-model}}

In this section, we will provide complete proof for Theorem \ref{thm:model-based}.
As a matter of fact, our proof strategy
here justifies a more general version of Theorem \ref{thm:model-based}
that accounts for model misspecification, as stated below. 

\begin{theorem}\label{thm:model-based-general}Suppose that $\delta>0$
and $\varepsilon\in(0,(1-\gamma)^{-1/2}]$. Assume that there exists
a probability transition model $\widetilde{\bm{P}}$ obeying Definition
\ref{eq:defn-linear-model} and Assumption \ref{assumption:anchor}
with feature vectors $\{\phi(s,a)\}_{(s,a)\in\mathcal{S}\times\mathcal{A}}\subset\mathbb{R}^K$ and anchor state-action pairs $\mathcal{K}$
such that $$\Vert\widetilde{\bm{P}}-\bm{P}\Vert_{1}\leq\xi$$ for some $\xi\geq0$. Let $\widehat{\pi}$ be
the policy returned by Algorithm \ref{alg:model-based-rl}. Assume
that
\begin{equation}
N\geq\frac{C\log\left(K/\left(\left(1-\gamma\right)\delta\right)\right)}{\left(1-\gamma\right)^{3}\varepsilon^{2}}\label{eq:model-sample-thm-general}
\end{equation}
for some sufficiently large constant $C>0$. Then with probability
exceeding $1-\delta$,
\begin{equation}
Q^{\star}\left(s,a\right)-Q^{\hat{\pi}}\left(s,a\right)\leq\varepsilon+\frac{4\varepsilon_{\mathsf{opt}}}{1-\gamma}+\frac{22\xi}{(1-\gamma)^2}, \label{eq:model-error-general}
\end{equation}
for every state-action pair $(s,a)\in\mathcal{S}\times\mathcal{A}$. 
\end{theorem}

Theorem \ref{thm:model-based-general} subsumes Theorem
\ref{thm:model-based} as a special case with $\xi=0$. The remainder of this section is devoted to proving Theorem \ref{thm:model-based-general}.

\subsection{Proof of Theorem \ref{thm:model-based-general}}

The error $\bm{Q}^{\widehat{\pi}}-\bm{Q}^{\star}$ can be decomposed
as
\begin{align}
\bm{Q}^{\widehat{\pi}}-\bm{Q}^{\star} & =\bm{Q}^{\widehat{\pi}}-\widehat{\bm{Q}}^{\widehat{\pi}}+\widehat{\bm{Q}}^{\widehat{\pi}}-\widehat{\bm{Q}}^{\star}+\widehat{\bm{Q}}^{\star}-\bm{Q}^{\star}\nonumber \\
 & \geq\bm{Q}^{\widehat{\pi}}-\widehat{\bm{Q}}^{\widehat{\pi}}+\widehat{\bm{Q}}^{\widehat{\pi}}-\widehat{\bm{Q}}^{\star}+\widehat{\bm{Q}}^{\pi^{\star}}-\bm{Q}^{\star}\nonumber \\
 & \geq-\left(\left\Vert \bm{Q}^{\widehat{\pi}}-\widehat{\bm{Q}}^{\widehat{\pi}}\right\Vert _{\infty}+\left\Vert \widehat{\bm{Q}}^{\widehat{\pi}}-\widehat{\bm{Q}}^{\star}\right\Vert _{\infty}+\left\Vert \widehat{\bm{Q}}^{\pi^{\star}}-\bm{Q}^{\star}\right\Vert _{\infty}\right)\bm{1}.\label{eq:model-q-0}
\end{align}
For policy $\widehat{\pi}$ satisfying the condition in Theorem \ref{thm:model-based},
we have $\Vert\widehat{\bm{Q}}^{\widehat{\pi}}-\widehat{\bm{Q}}^{\star}\Vert_{\infty}\leq\varepsilon_{\text{opt}}$.
It boils down to control $\Vert\bm{Q}^{\widehat{\pi}}-\widehat{\bm{Q}}^{\widehat{\pi}}\Vert_{\infty}$
and $\Vert\widehat{\bm{Q}}^{\pi^{\star}}-\bm{Q}^{\star}\Vert_{\infty}$.

To begin with, we can use \eqref{eq:equation-q} to
further decompose $\Vert\bm{Q}^{\widehat{\pi}}-\widehat{\bm{Q}}^{\widehat{\pi}}\Vert_{\infty}$
as
\begin{align}
\left\Vert \bm{Q}^{\widehat{\pi}}-\widehat{\bm{Q}}^{\widehat{\pi}}\right\Vert _{\infty} & =\left\Vert \left(\bm{I}-\gamma\bm{P}^{\widehat{\pi}}\right)^{-1}\bm{r}-\left(\bm{I}-\gamma\widehat{\bm{P}}^{\widehat{\pi}}\right)^{-1}\bm{r}\right\Vert _{\infty}\nonumber\\
 & =\left\Vert \left(\bm{I}-\gamma\bm{P}^{\widehat{\pi}}\right)^{-1}\left[\left(\bm{I}-\gamma\widehat{\bm{P}}^{\widehat{\pi}}\right)-\left(\bm{I}-\gamma\bm{P}^{\widehat{\pi}}\right)\right]\widehat{\bm{Q}}^{\widehat{\pi}}\right\Vert _{\infty}\nonumber\\
 & =\left\Vert \gamma\left(\bm{I}-\gamma\bm{P}^{\widehat{\pi}}\right)^{-1}\left(\bm{P}-\widehat{\bm{P}}\right)\widehat{\bm{V}}^{\widehat{\pi}}\right\Vert _{\infty}\nonumber \\
 & \leq\left\Vert \gamma\left(\bm{I}-\gamma\bm{P}^{\widehat{\pi}}\right)^{-1}\left(\bm{P}-\widehat{\bm{P}}\right)\widehat{\bm{V}}^{\star}\right\Vert _{\infty}+\left\Vert \gamma\left(\bm{I}-\gamma\bm{P}^{\widehat{\pi}}\right)^{-1}\left(\bm{P}-\widehat{\bm{P}}\right)\left(\widehat{\bm{V}}^{\widehat{\pi}}-\widehat{\bm{V}}^{\star}\right)\right\Vert _{\infty}\nonumber \\
 & \leq\left\Vert \gamma\left(\bm{I}-\gamma\bm{P}^{\widehat{\pi}}\right)^{-1}\left|\left(\bm{P}-\widehat{\bm{P}}\right)\widehat{\bm{V}}^{\star}\right|\right\Vert _{\infty}+\frac{2\gamma\varepsilon_{\text{opt}}}{1-\gamma}.\label{eq:model-qhat}
\end{align}
Here the last inequality is due to 
\begin{align*}
 & \left\Vert \gamma\left(\bm{I}-\gamma\bm{P}^{\widehat{\pi}}\right)^{-1}\left(\bm{P}-\widehat{\bm{P}}\right)\left(\widehat{\bm{V}}^{\widehat{\pi}}-\widehat{\bm{V}}^{\star}\right)\right\Vert _{\infty}\\
 & \quad\leq\gamma\left\Vert \left(\bm{I}-\gamma\bm{P}^{\widehat{\pi}}\right)^{-1}\right\Vert _{1}\left\Vert \left(\bm{P}-\widehat{\bm{P}}\right)\left(\widehat{\bm{V}}^{\widehat{\pi}}-\widehat{\bm{V}}^{\star}\right)\right\Vert _{\infty}\\
 & \quad\leq\gamma\left\Vert \left(\bm{I}-\gamma\bm{P}^{\widehat{\pi}}\right)^{-1}\right\Vert _{1}\left(\left\Vert \bm{P}\right\Vert _{1}+\left\Vert \widehat{\bm{P}}\right\Vert _{1}\right)\left\Vert \widehat{\bm{V}}^{\widehat{\pi}}-\widehat{\bm{V}}^{\star}\right\Vert _{\infty}\\
 & \quad\leq\frac{2\gamma\varepsilon_{\text{opt}}}{1-\gamma},
\end{align*}
where we use the fact that $\Vert(\bm{I}-\gamma\bm{P}^{\widehat{\pi}})^{-1}\Vert_{1}\leq1/(1-\gamma)$
and $\Vert\bm{P}\Vert_{1}=\Vert\widehat{\bm{P}}\Vert_{1}=1$.

Similarly, for the term $\Vert\widehat{\bm{Q}}^{\pi^{\star}}-\bm{Q}^{\star}\Vert_{\infty}$
in \eqref{eq:model-q-0}, we have 
\begin{equation}
\left\Vert \widehat{\bm{Q}}^{\pi^{\star}}-\bm{Q}^{\star}\right\Vert _{\infty}=\left\Vert \gamma\left(\bm{I}-\gamma\bm{P}^{\pi^{\star}}\right)^{-1}\left(\bm{P}-\widehat{\bm{P}}\right)\widehat{\bm{V}}^{\pi^{\star}}\right\Vert _{\infty}\leq\left\Vert \gamma\left(\bm{I}-\gamma\bm{P}^{\pi^{\star}}\right)^{-1}\left|\left(\bm{P}-\widehat{\bm{P}}\right)\widehat{\bm{V}}^{\pi^{\star}}\right|\right\Vert _{\infty}.\label{eq:model-qstar}
\end{equation}
As can be seen from \eqref{eq:model-qhat} and \eqref{eq:model-qstar},
it boils down to bound $|(\bm{P}-\widehat{\bm{P}})\widehat{\bm{V}}^{\star}|$
and $|(\bm{P}-\widehat{\bm{P}})\widehat{\bm{V}}^{\pi^{\star}}|$.
We have the following lemma.

\begin{lemma}\label{lemma:loo-model-based}With probability exceeding
$1-\delta$, one has 
\begin{align}
\left|\left(\bm{P}-\widehat{\bm{P}}\right)_{s,a}\widehat{\bm{V}}^{\star}\right| & \leq\frac{10\xi}{1-\gamma}+4\sqrt{\frac{2\log\left(4K/\delta\right)}{N}}+\frac{4\log\left(8K/\left(\left(1-\gamma\right)\delta\right)\right)}{\left(1-\gamma\right)N}\nonumber\\
&+\sqrt{\frac{4\log\left(8K/\left(\left(1-\gamma\right)\delta\right)\right)}{N}}\sqrt{\mathsf{Var}_{\bm{P}_{s,a}}\left(\widehat{\bm{V}}^{\star}\right)},\label{eq:loo-1}\\
\left|\left(\bm{P}-\widehat{\bm{P}}\right)_{s,a}\widehat{\bm{V}}^{\pi^{\star}}\right| & \leq\frac{10\xi}{1-\gamma}+4\sqrt{\frac{2\log\left(4K/\delta\right)}{N}}+\frac{4\log\left(8K/\left(\left(1-\gamma\right)\delta\right)\right)}{\left(1-\gamma\right)N}\nonumber\\
&+\sqrt{\frac{4\log\left(8K/\left(\left(1-\gamma\right)\delta\right)\right)}{N}}\sqrt{\mathsf{Var}_{\bm{P}_{s,a}}\left(\widehat{\bm{V}}^{\pi^{\star}}\right)}.\label{eq:loo-2}
\end{align}
\end{lemma}
\begin{proof}
	See Appendix \ref{subsec:Proof-of-Lemma-loo}.
\end{proof}

Applying \eqref{eq:loo-1} to \eqref{eq:model-qhat} reveals that 
\begin{align}
\left\Vert \bm{Q}^{\widehat{\pi}}-\widehat{\bm{Q}}^{\widehat{\pi}}\right\Vert _{\infty} & \leq\sqrt{\frac{4\log\left(8K/\left(\left(1-\gamma\right)\delta\right)\right)}{N}}\left\Vert \gamma\left(\bm{I}-\gamma\bm{P}^{\widehat{\pi}}\right)^{-1}\sqrt{\mathsf{Var}_{\bm{P}_{s,a}}\left(\widehat{\bm{V}}^{\star}\right)}\right\Vert _{\infty}\nonumber \\
 & \qquad+\frac{\gamma}{1-\gamma}\left[4\sqrt{\frac{2\log\left(4K/\delta\right)}{N}}+\frac{4\log\left(8K/\left(\left(1-\gamma\right)\delta\right)\right)}{\left(1-\gamma\right)N}\right]+\frac{10\gamma\xi}{(1-\gamma)^2}+\frac{2\gamma\varepsilon_{\text{opt}}}{1-\gamma}.\label{eq:model-q-hatpi}
\end{align}
For the first term, one has 
\begin{align*}
\sqrt{\mathsf{Var}_{\bm{P}_{s,a}}\left(\widehat{\bm{V}}^{\star}\right)} & \leq\sqrt{\mathsf{Var}_{\bm{P}_{s,a}}\left(\bm{V}^{\widehat{\pi}}\right)}+\sqrt{\mathsf{Var}_{\bm{P}_{s,a}}\left(\bm{V}^{\widehat{\pi}}-\widehat{\bm{V}}^{\widehat{\pi}}\right)}+\sqrt{\mathsf{Var}_{\bm{P}_{s,a}}\left(\widehat{\bm{V}}^{\widehat{\pi}}-\widehat{\bm{V}}^{\star}\right)}\\
 & \leq\sqrt{\mathsf{Var}_{\bm{P}_{s,a}}\left(\bm{V}^{\widehat{\pi}}\right)}+\left\Vert \bm{V}^{\widehat{\pi}}-\widehat{\bm{V}}^{\widehat{\pi}}\right\Vert _{\infty}+\varepsilon_{\text{opt}}\\
 & \leq\sqrt{\mathsf{Var}_{\bm{P}_{s,a}}\left(\bm{V}^{\widehat{\pi}}\right)}+\left\Vert \bm{Q}^{\widehat{\pi}}-\widehat{\bm{Q}}^{\widehat{\pi}}\right\Vert _{\infty}+\varepsilon_{\text{opt}},
\end{align*}
where the first inequality comes from the fact that $\sqrt{\mathsf{Var}(X+Y)}\leq\sqrt{\mathsf{Var}(X)}+\sqrt{\mathsf{Var}(Y)}$ for any random variables $X$ and $Y$.
It follows that 
\begin{align}
\left\Vert \gamma\left(\bm{I}-\gamma\bm{P}^{\widehat{\pi}}\right)^{-1}\sqrt{\mathsf{Var}_{\bm{P}_{s,a}}\left(\widehat{\bm{V}}^{\star}\right)}\right\Vert _{\infty} & \leq\left\Vert \gamma\left(\bm{I}-\gamma\bm{P}^{\widehat{\pi}}\right)^{-1}\sqrt{\mathsf{Var}_{\bm{P}_{s,a}}\left(\bm{V}^{\widehat{\pi}}\right)}\right\Vert _{\infty}+\frac{\gamma}{1-\gamma}\left(\left\Vert \bm{Q}^{\widehat{\pi}}-\widehat{\bm{Q}}^{\widehat{\pi}}\right\Vert _{\infty}+\varepsilon_{\text{opt}}\right)\nonumber \\
 & \leq\gamma\sqrt{\frac{2}{\left(1-\gamma\right)^{3}}}+\frac{\gamma}{1-\gamma}\left(\left\Vert \bm{Q}^{\widehat{\pi}}-\widehat{\bm{Q}}^{\widehat{\pi}}\right\Vert _{\infty}+\varepsilon_{\text{opt}}\right),\label{eq:model-qhat-sqrt-var}
\end{align}
where the second inequality utilizes \cite[Lemma 7]{azar2013minimax}.

Plugging \eqref{eq:model-qhat-sqrt-var} into \eqref{eq:model-q-hatpi}
yields 
\begin{align*}
\left\Vert \bm{Q}^{\widehat{\pi}}-\widehat{\bm{Q}}^{\widehat{\pi}}\right\Vert _{\infty} & \leq\sqrt{\frac{4\log\left(8K/\left(\left(1-\gamma\right)\delta\right)\right)}{N}}\left[\gamma\sqrt{\frac{2}{\left(1-\gamma\right)^{3}}}+\frac{\gamma}{1-\gamma}\left(\left\Vert \bm{Q}^{\widehat{\pi}}-\widehat{\bm{Q}}^{\widehat{\pi}}\right\Vert _{\infty}+\varepsilon_{\text{opt}}\right)\right]\\
 & \qquad+\frac{\gamma}{1-\gamma}\left[4\sqrt{\frac{2\log\left(4K/\delta\right)}{N}}+\frac{4\log\left(8K/\left(\left(1-\gamma\right)\delta\right)\right)}{\left(1-\gamma\right)N}\right]+\frac{10\gamma\xi}{(1-\gamma)^2}+\frac{2\gamma\varepsilon_{\text{opt}}}{1-\gamma}.
\end{align*}
Then we can rearrange terms to obtain
\begin{align}
\left\Vert \bm{Q}^{\widehat{\pi}}-\widehat{\bm{Q}}^{\widehat{\pi}}\right\Vert _{\infty} & \leq10\gamma\sqrt{\frac{\log\left(8K/\left(\left(1-\gamma\right)\delta\right)\right)}{N\left(1-\gamma\right)^{3}}}+\frac{11\gamma\xi}{(1-\gamma)^2}+\frac{3\gamma\varepsilon_{\text{opt}}}{1-\gamma}\label{eq:model-q-1}
\end{align}
as long as $N\geq \widetilde{C}\log(8K/((1-\gamma)\delta))/(1-\gamma)^{2}$ for some sufficiently large constant $\widetilde{C}>0$.

In a similar vein, we can use \eqref{eq:model-qstar} and \eqref{eq:loo-2}
to obtain that 
\begin{equation}
\left\Vert \widehat{\bm{Q}}^{\pi^{\star}}-\bm{Q}^{\star}\right\Vert _{\infty}\leq10\gamma\sqrt{\frac{\log\left(8K/\left(\left(1-\gamma\right)\delta\right)\right)}{N\left(1-\gamma\right)^{3}}}+\frac{11\gamma\xi}{(1-\gamma)^2}.\label{eq:model-q-2}
\end{equation}
Finally, we can substitute \eqref{eq:model-q-1} and \eqref{eq:model-q-2}
into \eqref{eq:model-q-0} to achieve
\[
\bm{Q}^{\widehat{\pi}}-\bm{Q}^{\star}\geq-\left(20\gamma\sqrt{\frac{\log\left(8K/\left(\left(1-\gamma\right)\delta\right)\right)}{N\left(1-\gamma\right)^{3}}}+\frac{22\gamma\xi}{(1-\gamma)^2}+\frac{3\gamma\varepsilon_{\text{opt}}}{1-\gamma}+\varepsilon_{\text{opt}}\right)\bm{1}.
\]
This result implies that 
\[
\bm{Q}^{\widehat{\pi}}\geq\bm{Q}^{\star}-\left(\varepsilon+\frac{22\xi}{(1-\gamma)^2}+\frac{4\varepsilon_{\text{opt}}}{1-\gamma}\right)\bm{1},
\]
as long as 
\[
N\geq\frac{C\log\left(8K/\left(\left(1-\gamma\right)\delta\right)\right)}{\left(1-\gamma\right)^{3}\varepsilon^{2}},
\]
for some sufficiently large constant $C>0$.

\subsection{Proof of Lemma \ref{lemma:loo-model-based}\label{subsec:Proof-of-Lemma-loo}}
To prove this theorem, we invoke the idea of $s$-absorbing MDP proposed
by \cite{agarwal2020model}. For a state $s\in\mathcal{S}$ and a scalar $u$, we
define a new MDP $M_{s,u}$ to be identical to $M$ on all the other
states except $s$; on state $s$, $M_{s,u}$ is absorbing such that
$P_{M_{s,u}}(s|s,a)=1$ and $r_{M_{s,u}}\left(s,a\right)=(1-\gamma)u$
for all $a\in\mathcal{A}$. More formally, we define $P_{M_{u,s}}$ and $r_{M_{u,s}}$ as
\begin{align*}
P_{M_{s,u}}(s|s,a)=1,\quad r_{M_{s,u}}\left(s,a\right)=(1-\gamma)u, \qquad &\text{for all }a\in\mathcal{A}, \\ 
P_{M_{s,u}}(\cdot|s',a')=P(\cdot|s',a'),\quad r_{M_{s,u}}\left(s,a\right)=r\left(s,a\right), \qquad &\text{for all }s'\neq s\text{ and }a'\in\mathcal{A}. 
\end{align*}
To streamline notations, we will use
$\bm{V}_{s,u}^{\pi}\in\mathbb{R}^{|\mathcal{S}|}$ and $\bm{V}_{s,u}^{\star}\in\mathbb{R}^{|\mathcal{S}|}$
to denote the value function of $M_{s,u}$ under policy $\pi$ and
the optimal value function of $M_{s,u}$ respectively. Furthermore,
we denote by $\widehat{M}_{s,u}$ the MDP whose probability transition kernel is identical to $\widehat{P}$ at all states except that state $s$ is absorbing. Similar
as before, we use $\widehat{\bm{V}}_{s,u}^{\star}\in\mathbb{R}^{|\mathcal{S}|}$
to denote the optimal value function under $\widehat{M}_{s,u}$. The
construction of this collection of auxiliary MDPs will facilitate our analysis
by decoupling the statistical dependency between $\widehat{\bm{P}}$ and $\widehat{\pi}^\star$.

To begin with, we can decompose the quantity of interest as 
\begin{align}
\left|\left(\bm{P}-\widehat{\bm{P}}\right)_{s,a}\widehat{\bm{V}}^{\star}\right| & =\left|\left(\bm{P}-\widehat{\bm{P}}\right)_{s,a}\left(\widehat{\bm{V}}^{\star}-\widehat{\bm{V}}_{s,u}^{\star}+\widehat{\bm{V}}_{s,u}^{\star}\right)\right|\nonumber \\
 & \leq\left|\left(\bm{P}-\widehat{\bm{P}}\right)_{s,a}\widehat{\bm{V}}_{s,u}^{\star}\right|+\left|\left(\bm{P}-\widehat{\bm{P}}\right)_{s,a}\left(\widehat{\bm{V}}^{\star}-\widehat{\bm{V}}_{s,u}^{\star}\right)\right|\nonumber \\
 & \overset{(\text{i})}{\leq}\left|\left(\bm{P}-\widetilde{\bm{P}}\right)_{s,a}\widehat{\bm{V}}_{s,u}^{\star}\right|+\left|\bm{\lambda}\left(s,a\right)\left(\widetilde{\bm{P}}_{\mathcal{K}}-\bm{P}_{\mathcal{K}}\right)\widehat{\bm{V}}_{s,u}^{\star}\right|\nonumber \\
 & \qquad+\left|\bm{\lambda}\left(s,a\right)\left(\bm{P}_{\mathcal{K}}-\widehat{\bm{P}}_{\mathcal{K}}\right)\widehat{\bm{V}}_{s,u}^{\star}\right|+\left(\left\Vert \bm{P}_{s,a}\right\Vert _{1}+\left\Vert \widehat{\bm{P}}_{s,a}\right\Vert _{1}\right)\left\Vert \widehat{\bm{V}}^{\star}-\widehat{\bm{V}}_{s,u}^{\star}\right\Vert _{\infty}\\
 & \leq\left\Vert \left(\bm{P}-\widetilde{\bm{P}}\right)_{s,a}\right\Vert _{1}\left\Vert \widehat{\bm{V}}_{s,u}^{\star}\right\Vert _{\infty}+\left\Vert \bm{\lambda}\left(s,a\right)\right\Vert _{1}\cdot\left\Vert \left(\widetilde{\bm{P}}_{\mathcal{K}}-\bm{P}_{\mathcal{K}}\right)\widehat{\bm{V}}_{s,u}^{\star}\right\Vert _{\infty}\nonumber \\
 & \qquad+\left\Vert \bm{\lambda}\left(s,a\right)\right\Vert _{1}\cdot\left\Vert \left(\bm{P}_{\mathcal{K}}-\widehat{\bm{P}}_{\mathcal{K}}\right)\widehat{\bm{V}}_{s,u}^{\star}\right\Vert _{\infty}+2\left\Vert \widehat{\bm{V}}^{\star}-\widehat{\bm{V}}_{s,u}^{\star}\right\Vert _{\infty}\\
 & \overset{(\text{ii})}{\leq}\frac{2\xi}{1-\gamma}+\max_{\left(s,a\right)\in\mathcal{K}}\left|\left(\bm{P}-\widehat{\bm{P}}\right)_{s,a}\widehat{\bm{V}}_{s,u}^{\star}\right|+2\left\Vert \widehat{\bm{V}}^{\star}-\widehat{\bm{V}}_{s,u}^{\star}\right\Vert _{\infty},\label{eq:model-pk}
\end{align}
where (i) makes use of $\widetilde{\bm{P}}_{s,a}=\bm{\lambda}(s,a)\widetilde{\bm{P}}_{\mathcal{K}}$
and $\widehat{\bm{P}}_{s,a}=\bm{\lambda}(s,a)\widehat{\bm{P}}_{\mathcal{K}}$;
(ii) depends on $\Vert\bm{P}-\widetilde{\bm{P}}\Vert_{1}\leq\xi$,
$\Vert\bm{\lambda}(s,a)\Vert_{1}=1$ and $\Vert\widehat{\bm{V}}_{s,u}^{\star}\Vert_{\infty}\leq(1-\gamma)^{-1}$.
For each state $s$, the value of $u$ will be selected from a set $\mathcal{U}_s$. The choice of $\mathcal{U}_s$ will be specified later. Then for some fixed $u$ in $\mathcal{U}_s$ and fixed state-action pair $(s,a)\in\mathcal{K}$, 
due to the independence between $\widehat{\bm{P}}_{s,a}$ and $\widehat{\bm{V}}_{s,u}^{\star}$,
we can apply Bernstein's inequality (cf.~\cite[Theorem 2.8.4]{vershynin2018high}) conditional on $\widehat{\bm{V}}_{s,u}^{\star}$ to reveal that with probability greater than $1-\delta/2$, 
\begin{equation}
\left|\left(\bm{P}-\widehat{\bm{P}}\right)_{s,a}\widehat{\bm{V}}_{s,u}^{\star}\right|\leq\sqrt{\frac{2\log\left(4/\delta\right)}{N}\mathsf{Var}_{\bm{P}_{s,a}}\left(\widehat{\bm{V}}_{s,u}^{\star}\right)}+\frac{2\log\left(4/\delta\right)}{3\left(1-\gamma\right)N}.\label{eq:model-pv-loo-fix}
\end{equation}
Invoking the union bound over all the $K$ state-action pairs of $\mathcal{K}$ and all the possible values of $u$ in $\mathcal{U}_s$ demonstrate that with probability greater than $1-\delta/2$, 
\begin{equation}
\left|\left(\bm{P}-\widehat{\bm{P}}\right)_{s,a}\widehat{\bm{V}}_{s,u}^{\star}\right|\leq\sqrt{\frac{2\log\left(4K\left|\mathcal{U}_{s}\right|/\delta\right)}{N}\mathsf{Var}_{\bm{P}_{s,a}}\left(\widehat{\bm{V}}_{s,u}^{\star}\right)}+\frac{2\log\left(4K\left|\mathcal{U}_{s}\right|/\delta\right)}{3\left(1-\gamma\right)N},\label{eq:model-pv-loo}
\end{equation}
holds for all state-action pair $(s,a)\in\mathcal{K}$ and all $u\in \mathcal{U}_{s}$.
Here, $\mathsf{Var}_{\bm{P}_{s,a}}(\cdot)$ is defined in \eqref{eq:defn-varpsav}.
Then we observe that 
\begin{align}
\sqrt{\mathsf{Var}_{\bm{P}_{s,a}}\left(\widehat{\bm{V}}_{s,u}^{\star}\right)} & \leq\sqrt{\mathsf{Var}_{\bm{P}_{s,a}}\left(\widehat{\bm{V}}^{\star}-\widehat{\bm{V}}_{s,u}^{\star}\right)}+\sqrt{\mathsf{Var}_{\bm{P}_{s,a}}\left(\widehat{\bm{V}}^{\star}\right)}\nonumber \\
 & \leq\left\Vert \widehat{\bm{V}}^{\star}-\widehat{\bm{V}}_{s,u}^{\star}\right\Vert _{\infty}+\sqrt{\mathsf{Var}_{\bm{P}_{s,a}}\left(\widehat{\bm{V}}^{\star}\right)}\nonumber \\
 & \leq\left|\widehat{\bm{V}}^{\star}\left(s\right)-u\right|+\sqrt{\mathsf{Var}_{\bm{P}_{s,a}}\left(\widehat{\bm{V}}^{\star}\right)},\label{eq:model-var-loo}
\end{align}
where (i) is due to $\sqrt{\mathsf{Var}_{\bm{P}_{s,a}}(\bm{V}_{1}+\bm{V}_{2})}\leq\sqrt{\mathsf{Var}_{\bm{P}_{s,a}}(\bm{V}_{1})}+\sqrt{\mathsf{Var}_{\bm{P}_{s,a}}(\bm{V}_{2})}$
and (ii) holds since
\begin{equation}
\left\Vert \widehat{\bm{V}}^{\star}-\widehat{\bm{V}}_{s,u}^{\star}\right\Vert _{\infty}=\left\Vert \widehat{\bm{V}}_{s,\widehat{\bm{V}}^{\star}\left(s\right)}^{\star}-\widehat{\bm{V}}_{s,u}^{\star}\right\Vert _{\infty}\leq\left|\widehat{\bm{V}}^{\star}\left(s\right)-u\right|,\label{eq:v-loo}
\end{equation}whose proof can be found in \cite[Lemma 8 and 9]{agarwal2020model}.

By substituting \eqref{eq:model-pv-loo}, \eqref{eq:model-var-loo}
and \eqref{eq:v-loo} into \eqref{eq:model-pk}, we arrive at
\begin{align}
\left|\left(\bm{P}-\widehat{\bm{P}}\right)_{s,a}\widehat{\bm{V}}^{\star}\right| & \leq\frac{2\xi}{1-\gamma}+\left|\widehat{\bm{V}}^{\star}\left(s\right)-u\right|\left(2+\sqrt{\frac{2\log\left(4K\left|\mathcal{U}_{s}\right|/\delta\right)}{N}}\right)\nonumber\\
 & \qquad+\sqrt{\frac{2\log\left(4K\left|\mathcal{U}_{s}\right|/\delta\right)}{N}}\sqrt{\mathsf{Var}_{\bm{P}_{s,a}}\left(\widehat{\bm{V}}^{\star}\right)}+\frac{2\log\left(4K\left|\mathcal{U}_{s}\right|/\delta\right)}{3\left(1-\gamma\right)N}.\label{eq:p-1}
\end{align}
Then it boils down to determining $\mathcal{U}_{s}$. The coarse bounds of $\widehat{\bm{Q}}^{\pi^{\star}}$
and $\widehat{\bm{Q}}^{\star}$ in the following lemma provide a guidance
on the choice of $\mathcal{U}_{s}$.

\begin{lemma}\label{lemma:coarse}For $\delta\in(0,1)$, with probability exceeding $1-\delta/2$ one has 
\begin{align}
\left\Vert \bm{Q}^{\star}-\widehat{\bm{Q}}^{\pi^{\star}}\right\Vert _{\infty}&\leq\frac{\gamma}{1-\gamma}\sqrt{\frac{\log\left(4K/\delta\right)}{2N\left(1-\gamma\right)^{2}}}+\frac{2\gamma\xi}{\left(1-\gamma\right)^{2}},\label{eq:crude-1}\\
\left\Vert \bm{Q}^{\star}-\widehat{\bm{Q}}^{\star}\right\Vert _{\infty}&\leq\frac{\gamma}{1-\gamma}\sqrt{\frac{\log\left(4K/\delta\right)}{2N\left(1-\gamma\right)^{2}}}+\frac{2\gamma\xi}{\left(1-\gamma\right)^{2}}. \label{eq:crude-2} 
\end{align}
\end{lemma}
\begin{proof}
	See Appendix \ref{subsec:Proof-of-Lemma-coarse}.
\end{proof}

This inspires us to choose $\mathcal{U}_{s}$ to be the set consisting of equidistant
points in $[\bm{V}^{\star}(s)-R(\delta),\bm{V}^{\star}(s)+R(\delta)]$
with $|U_{s}|=\left\lceil 1/(1-\gamma)^{2}\right\rceil $ and 
\[
R\left(\delta\right)\coloneqq\frac{\gamma}{1-\gamma}\sqrt{\frac{\log\left(4K/\delta\right)}{2N\left(1-\gamma\right)^{2}}}+\frac{2\gamma\xi}{\left(1-\gamma\right)^{2}}.
\]
Since $\Vert\bm{V}^{\star}-\widehat{\bm{V}}^{\star}\Vert_{\infty}\leq\Vert\bm{Q}^{\star}-\widehat{\bm{Q}}^{\star}\Vert_{\infty}$,
Lemma \ref{lemma:coarse} implies that $\widehat{\bm{V}}^{\star}(s)\in[\bm{V}^{\star}(s)-R(\delta),\bm{V}^{\star}(s)+R(\delta)]$
with probability over $1-\delta/2$. Hence, we have 
\begin{equation}
	\min_{u\in \mathcal{U}_{s}}\left|\widehat{\bm{V}}^{\star}\left(s\right)-u\right|\leq\frac{2R\left(\delta\right)}{\left|U_{s}\right|+1}\leq2\gamma\sqrt{\frac{2\log\left(4K/\delta\right)}{N}}+4\gamma\xi.\label{eq:p-2}
\end{equation}
Consequently, with probability exceeding $1-\delta$, one has 
\begin{align*}
\left|\left(\bm{P}-\widehat{\bm{P}}\right)_{s,a}\widehat{\bm{V}}^{\star}\right| & \overset{\text{(i)}}{\leq}\frac{2\xi}{1-\gamma}+\min_{u\in \mathcal{U}_{s}}\left|\widehat{\bm{V}}^{\star}\left(s\right)-u\right|\left(2+\sqrt{\frac{2\log\left(4K\left|\mathcal{U}_{s}\right|/\delta\right)}{N}}\right)\\
 & \qquad+\sqrt{\frac{2\log\left(4K\left|\mathcal{U}_{s}\right|/\delta\right)}{N}}\sqrt{\mathsf{Var}_{\bm{P}_{s,a}}\left(\widehat{\bm{V}}^{\star}\right)}+\frac{2\log\left(4K\left|\mathcal{U}_{s}\right|/\delta\right)}{3\left(1-\gamma\right)N}\\
 & \overset{\text{(ii)}}{\leq}\frac{2\xi}{1-\gamma}+\left(2\gamma\sqrt{\frac{2\log\left(4K/\delta\right)}{N}}+4\gamma\xi\right)\left(2+\sqrt{\frac{4\log\left(8K/\left(\left(1-\gamma\right)\delta\right)\right)}{N}}\right)\\
 & \qquad+\sqrt{\frac{4\log\left(8K/\left(\left(1-\gamma\right)\delta\right)\right)}{N}}\sqrt{\mathsf{Var}_{\bm{P}_{s,a}}\left(\widehat{\bm{V}}^{\star}\right)}+\frac{2\log\left(8K/\left(\left(1-\gamma\right)\delta\right)\right)}{3\left(1-\gamma\right)N}\\
 & \leq\frac{10\xi}{1-\gamma}+4\sqrt{\frac{2\log\left(4K/\delta\right)}{N}}+\frac{4\log\left(8K/\left(\left(1-\gamma\right)\delta\right)\right)}{\left(1-\gamma\right)N}\\
 & \qquad+\sqrt{\frac{4\log\left(8K/\left(\left(1-\gamma\right)\delta\right)\right)}{N}}\sqrt{\mathsf{Var}_{\bm{P}_{s,a}}\left(\widehat{\bm{V}}^{\star}\right)},
\end{align*}
where (i) follows from \eqref{eq:p-1} and (ii) utilizes \eqref{eq:p-2}. 
This finishes the proof for the first inequality. The second inequality can be proved in a similar way and is omitted here for
brevity. 

\subsection{Proof of Lemma \ref{lemma:coarse}\label{subsec:Proof-of-Lemma-coarse}}

To begin with, one has 
\begin{align}
\left\Vert \left(\widehat{\bm{P}}-\bm{P}\right)\bm{V}^{\star}\right\Vert _{\infty} & \leq\left\Vert \bm{\Lambda}\left(\widehat{\bm{P}}_{\mathcal{K}}-\bm{P}_{\mathcal{K}}\right)\bm{V}^{\star}\right\Vert _{\infty}+\left\Vert \bm{\Lambda}\left(\bm{P}_{\mathcal{K}}-\widetilde{\bm{P}}_{\mathcal{K}}\right)\bm{V}^{\star}\right\Vert _{\infty}+\left\Vert \left(\widetilde{\bm{P}}-\bm{P}\right)\bm{V}^{\star}\right\Vert _{\infty}\nonumber \\
 & \leq\left\Vert \bm{\Lambda}\right\Vert _{1}\left\Vert \left(\widehat{\bm{P}}_{\mathcal{K}}-\bm{P}_{\mathcal{K}}\right)\bm{V}^{\star}\right\Vert _{\infty}+\left\Vert \bm{\Lambda}\right\Vert _{1}\left\Vert \left(\bm{P}_{\mathcal{K}}-\widetilde{\bm{P}}_{\mathcal{K}}\right)\bm{V}^{\star}\right\Vert _{\infty}+\left\Vert \widetilde{\bm{P}}-\bm{P}\right\Vert _{1}\left\Vert \bm{V}^{\star}\right\Vert _{\infty}\nonumber \\
 & \leq\left\Vert \left(\widehat{\bm{P}}_{\mathcal{K}}-\bm{P}_{\mathcal{K}}\right)\bm{V}^{\star}\right\Vert _{\infty}+\frac{2\xi}{1-\gamma},\label{eq:coarse-pv}
\end{align}
where the first line uses $\widehat{\bm{P}}=\bm{\Lambda}\widehat{\bm{P}}_{\mathcal{K}}$
and $\widetilde{\bm{P}}=\bm{\Lambda}\widetilde{\bm{P}}_{\mathcal{K}}$;
the last inequality comes from the facts that $\Vert\widetilde{\bm{P}}-\bm{P}\Vert_{1}\leq\xi$,
$\Vert\bm{\Lambda}\Vert_{1}=1$ and $\Vert\bm{V}^{\star}\Vert_{\infty}\leq(1-\gamma)^{-1}$.
Then we turn to bound $\Vert(\widehat{\bm{P}}_{\mathcal{K}}-\bm{P}_{\mathcal{K}})\bm{V}^{\star}\Vert_{\infty}$.
In view of \eqref{eq:empirical-transition}, Hoeffding's inequality (cf.~\cite[Theorem 2.2.6]{vershynin2018high})
implies that for $(s,a)\in\mathcal{K}$, 
\[
\mathbb{P}\left(\left|\left(\widehat{\bm{P}}-\bm{P}\right)_{s,a}\bm{V}^{\star}\right|\geq t\right)\leq2\exp\left(-\frac{2t^{2}}{\left\Vert \bm{V}^{\star}\right\Vert _{\infty}^{2}/N}\right).
\]
Hence by the standard union bound argument we have

\begin{equation}
\left\Vert \left(\widehat{\bm{P}}_{\mathcal{K}}-\bm{P}_{\mathcal{K}}\right)\bm{V}^{\star}\right\Vert _{\infty}\leq\sqrt{\frac{\left\Vert \bm{V}^{\star}\right\Vert _{\infty}^{2}\log\left(4K/\delta\right)}{2N}}\leq\sqrt{\frac{\log\left(4K/\delta\right)}{2N\left(1-\gamma\right)^{2}}},\label{eq:coarse-pkv}
\end{equation}with probability over $1-\delta/2$. 

\begin{enumerate}
\item Now we are ready to bound $\bm{Q}^{\pi^{\star}}-\widehat{\bm{Q}}^{\pi^{\star}}$.
One has 
\begin{align*}
\bm{Q}^{\pi^{\star}}-\widehat{\bm{Q}}^{\pi^{\star}} & =\left(\bm{I}-\gamma\bm{P}^{\pi^{\star}}\right)^{-1}\bm{r}-\left(\bm{I}-\gamma\widehat{\bm{P}}^{\pi^{\star}}\right)^{-1}\bm{r}\\
 & =\left(\bm{I}-\gamma\widehat{\bm{P}}^{\pi^{\star}}\right)^{-1}\left(\left(\bm{I}-\gamma\widehat{\bm{P}}^{\pi^{\star}}\right)-\left(\bm{I}-\gamma\bm{P}^{\pi^{\star}}\right)\right)\bm{Q}^{\pi^{\star}}\\
 & =\gamma\left(\bm{I}-\gamma\widehat{\bm{P}}^{\pi^{\star}}\right)^{-1}\left(\bm{P}^{\pi^{\star}}-\widehat{\bm{P}}^{\pi^{\star}}\right)\bm{Q}^{\pi^{\star}}\\
 & =\gamma\left(\bm{I}-\gamma\widehat{\bm{P}}^{\pi^{\star}}\right)^{-1}\left(\bm{P}-\widehat{\bm{P}}\right)\bm{V}^{\pi^{\star}},
\end{align*}
where the first equality makes use of \eqref{eq:equation-q}. Then we
take \eqref{eq:coarse-pv} and \eqref{eq:coarse-pkv} collectively
to achieve 
\begin{align*}
\left\Vert \gamma\left(\bm{I}-\gamma\widehat{\bm{P}}^{\pi^{\star}}\right)^{-1}\left(\bm{P}-\widehat{\bm{P}}\right)\bm{V}^{\star}\right\Vert _{\infty} & \le\gamma\sum_{i=0}^{\infty}\left\Vert \gamma^{i}\left(\widehat{\bm{P}}^{\pi^{\star}}\right)^{i}\left(\bm{P}-\widehat{\bm{P}}\right)\bm{V}^{\star}\right\Vert _{\infty}\\
 & \leq\gamma\sum_{i=0}^{\infty}\gamma^{i}\left\Vert \left(\widehat{\bm{P}}^{\pi^{\star}}\right)^{i}\right\Vert _{1}\left\Vert \left(\bm{P}-\widehat{\bm{P}}\right)\bm{V}^{\star}\right\Vert _{\infty}\\
 & \leq\frac{\gamma}{1-\gamma}\sqrt{\frac{\log\left(4K/\delta\right)}{2N\left(1-\gamma\right)^{2}}}+\frac{2\gamma\xi}{\left(1-\gamma\right)^{2}},
\end{align*}
where the last line comes from the fact that for all $i\geq1$, $(\widehat{\bm{P}}^{\pi^{\star}})^{i}$
is a probability transition matrix so that $\Vert(\widehat{\bm{P}}^{\pi^{\star}})^{i}\Vert_{1}=1$.
This justifies the first inequality \eqref{eq:crude-1}. 
\item In terms of the second one, \cite[Section A.4]{agarwal2020model}
implies that 
\[
\left\Vert \bm{Q}^{\star}-\widehat{\bm{Q}}^{\star}\right\Vert _{\infty}\leq\frac{\gamma}{1-\gamma}\left\Vert \left(\bm{P}-\widehat{\bm{P}}\right)\bm{V}^{\star}\right\Vert _{\infty}.
\]
Substitution of \eqref{eq:coarse-pv} and \eqref{eq:coarse-pkv} into
the above inequality yields
\[
\left\Vert \bm{Q}^{\star}-\widehat{\bm{Q}}^{\star}\right\Vert _{\infty}\leq\frac{\gamma}{1-\gamma}\sqrt{\frac{\log\left(4K/\delta\right)}{2N\left(1-\gamma\right)^{2}}}+\frac{2\gamma\xi}{\left(1-\gamma\right)^{2}}.
\]
\end{enumerate}

\section{Analysis of Q-learning\label{subsec:Proof-of-Theorem-syn-q} (Proof of Theorem \ref{thm:syn-q-upper})}

In this section, we will provide complete proof for Theorem \ref{thm:syn-q-upper}.
We actually prove a more general version of Theorem \ref{thm:syn-q-upper}
that takes model misspecification into consideration, as stated below. 

\begin{theorem}\label{thm:syn-q-upper-general} Consider any $\delta\in(0,1)$
and $\varepsilon\in(0,1]$. Suppose that there exists
a probability transition model $\widetilde{\bm{P}}$ obeying Definition
\ref{eq:defn-linear-model} and Assumption \ref{assumption:anchor}
with feature vectors $\{\phi(s,a)\}_{(s,a)\in\mathcal{S}\times\mathcal{A}}\subset\mathbb{R}^K$ and anchor state-action pairs $\mathcal{K}$
such that $$\Vert\widetilde{\bm{P}}-\bm{P}\Vert_{1}\leq\xi$$ for some $\xi\geq0$. Assume that the initialization obeys $0\leq Q_{0}(s,a)\leq\frac{1}{1-\gamma}$
for any $(s,a)\in\mathcal{S}\times\mathcal{A}$ and for any $0\leq t\leq T$,
the learning rates satisfy 
\begin{equation}
\frac{1}{1+\frac{c_{1}\left(1-\gamma\right)T}{\log^{2}T}}\leq\eta_{t}\leq\frac{1}{1+\frac{c_{2}\left(1-\gamma\right)t}{\log^{2}T}},\label{eq:syn-q-stepsize-general}
\end{equation}
for some sufficiently small universal constants $c_{1}\geq c_{2}>0$.
Suppose that the total number of iterations $T$ exceeds 
\begin{equation}
T\geq\frac{C_{3}\log\left(KT/\delta\right)\log^{4}T}{\left(1-\gamma\right)^{4}\varepsilon^{2}},\label{eq:syn-q-iteration-complexity-general}
\end{equation}
for some sufficiently large universal constant $C_{3}>0$. If there
exists a linear probability transition model $\widetilde{\bm{P}}$
satisfying Assumption \ref{assumption:anchor} with feature vectors
$\{\phi(s,a)\}_{(s,a)\in\mathcal{S}\times\mathcal{A}}$ such that
$\Vert\widetilde{\bm{P}}-\bm{P}\Vert_{1}\leq\xi$, then with probability
exceeding $1-\delta$, the output $Q_{T}$ of Algorithm \ref{alg:syn-q}
satisfies
\begin{equation}
\max_{\left(s,a\right)\in\mathcal{S}\times\mathcal{A}}\left|Q_{T}\left(s,a\right)-Q^{\star}\left(s,a\right)\right|\leq\varepsilon+\frac{6\gamma\xi}{\left(1-\gamma\right)^{2}},\label{eq:model-free-Q-T-general}
\end{equation}
for some constant $C_4>0$. In addition, let $\pi_{T}$ (resp.~$V_{T}$)
to be the policy (resp.~value function) induced by $Q_{T}$, then
one has
\begin{equation}
	\max_{s\in\mathcal{S}}\left|V^{\pi_{T}}\left(s\right)-V^{\star}\left(s\right)\right|\leq\frac{2\gamma}{1-\gamma}\left(\varepsilon+\frac{6\gamma\xi}{\left(1-\gamma\right)^{2}}\right). \label{eq:model-free-V}
\end{equation}
\end{theorem}

Theorem \ref{thm:syn-q-upper-general} subsumes Theorem \ref{thm:syn-q-upper}
as a special case with $\xi=0$. The remainder of this section is devoted to proving Theorem \ref{thm:syn-q-upper-general}.

\subsection{Proof of Theorem \ref{thm:syn-q-upper-general}}
First we show that \eqref{eq:model-free-V} can be easily obtained from \eqref{eq:model-free-Q-T-general}. Since \cite{singh1994upper} gives rise to $$\Vert V^{\pi_{T}}-V^{\star}\Vert_{\infty}\leq\frac{2\gamma\Vert V_{T}-V^{\star}\Vert_{\infty}}{1-\gamma},$$
we have $$\Vert V^{\pi_{T}}-V^{\star}\Vert_{\infty}\leq\frac{2\gamma\Vert Q_{T}-Q^{\star}\Vert_{\infty}}{1-\gamma},$$
due to $\Vert V_{T}-V^{\star}\Vert_{\infty}\leq\Vert Q_{T}-Q^{\star}\Vert_{\infty}$. Then \eqref{eq:model-free-V} follows directly from \eqref{eq:model-free-Q-T-general}. 

Therefore, we are left to justify \eqref{eq:model-free-Q-T-general}. To start with, we consider the update rule

\[
\bm{Q}_{t}=\left(1-\eta_{t}\right)\bm{Q}_{t-1}+\eta_{t}\left(\bm{r}+\gamma\widehat{\bm{P}}_{t}\bm{V}_{t-1}\right).
\]
By defining the error term $\bm{\Delta}_{t}\coloneqq\bm{Q}_{t}-\bm{Q}^{\star}$,
we can decompose $\bm{\Delta}_{t}$ into 
\begin{align}
\bm{\Delta}_{t} & =\left(1-\eta_{t}\right)\bm{Q}_{t-1}+\eta_{t}\left(\bm{r}+\gamma\widehat{\bm{P}}_{t}\bm{V}_{t-1}\right)-\bm{Q}^{\star}\nonumber \\
 & =\left(1-\eta_{t}\right)\left(\bm{Q}_{t-1}-\bm{Q}^{\star}\right)+\eta_{t}\left(\bm{r}+\gamma\widehat{\bm{P}}_{t}\bm{V}_{t-1}-\bm{Q}^{\star}\right)\nonumber \\
 & =\left(1-\eta_{t}\right)\left(\bm{Q}_{t-1}-\bm{Q}^{\star}\right)+\gamma\eta_{t}\left(\widehat{\bm{P}}_{t}\bm{V}_{t-1}-\bm{P}\bm{V}^{\star}\right)\nonumber \\
 & =\left(1-\eta_{t}\right)\bm{\Delta}_{t-1}+\gamma\eta_{t}\bm{\Lambda}\left(\widehat{\bm{P}}_{\mathcal{K}}^{\left(t\right)}-\bm{P}_{\mathcal{K}}\right)\bm{V}_{t-1}+\gamma\eta_{t}\bm{\Lambda}\bm{P}_{\mathcal{K}}\left(\bm{V}_{t-1}-\bm{V}^{\star}\right)+\gamma\eta_{t}\left(\bm{\Lambda}\bm{P}_{\mathcal{K}}-\bm{P}\right)\bm{V}^{\star}.\label{eq:delta-decompose}
\end{align}
Here in the penultimate equality, we make use of $\bm{Q}^{\star}=\bm{r}+\gamma\bm{P}\bm{V}^{\star}$;
and the last equality comes from $\widehat{\bm{P}}_{t}=\bm{\Lambda}\widehat{\bm{P}}_{\mathcal{K}}^{(t)}$
which is defined in \eqref{eq:defn-PKt}. It is straightforward to
check that $\bm{\Lambda}\bm{P}_{\mathcal{K}}$ is also a probability
transition matrix. We denote by $\overline{\bm{P}}=\bm{\Lambda}\bm{P}_{\mathcal{K}}$
hereafter. The third term in the decomposition above can be upper and 
lower bounded by
\begin{align*}
\overline{\bm{P}}\left(\bm{V}_{t-1}-\bm{V}^{\star}\right) & =\overline{\bm{P}}^{\pi_{t-1}}\bm{Q}_{t-1}-\overline{\bm{P}}^{\pi^{\star}}\bm{Q}^{\star}\leq\overline{\bm{P}}^{\pi_{t-1}}\bm{Q}_{t-1}-\overline{\bm{P}}^{\pi_{t-1}}\bm{Q}^{\star}=\overline{\bm{P}}^{\pi_{t-1}}\bm{\Delta}_{t-1},
\end{align*}
and 
\[
\overline{\bm{P}}\left(\bm{V}_{t-1}-\bm{V}^{\star}\right)=\overline{\bm{P}}^{\pi_{t-1}}\bm{Q}_{t-1}-\overline{\bm{P}}^{\pi^{\star}}\bm{Q}^{\star}\geq\overline{\bm{P}}^{\pi^{\star}}\bm{Q}_{t-1}-\overline{\bm{P}}^{\pi^{\star}}\bm{Q}^{\star}=\overline{\bm{P}}^{\pi^{\star}}\bm{\Delta}_{t-1}.
\]
Plugging these bounds into \eqref{eq:delta-decompose} yields 
\begin{align*}
\bm{\Delta}_{t} & \leq\left(1-\eta_{t}\right)\bm{\Delta}_{t-1}+\gamma\eta_{t}\bm{\Lambda}\left(\widehat{\bm{P}}_{\mathcal{K}}^{\left(t\right)}-\bm{P}_{\mathcal{K}}\right)\bm{V}_{t-1}+\gamma\eta_{t}\overline{\bm{P}}^{\pi_{t-1}}\bm{\Delta}_{t-1}+\gamma\eta_{t}\left(\bm{\Lambda}\bm{P}_{\mathcal{K}}-\bm{P}\right)\bm{V}^{\star},\\
\bm{\Delta}_{t} & \geq\left(1-\eta_{t}\right)\bm{\Delta}_{t-1}+\gamma\eta_{t}\bm{\Lambda}\left(\widehat{\bm{P}}_{\mathcal{K}}^{\left(t\right)}-\bm{P}_{\mathcal{K}}\right)\bm{V}_{t-1}+\gamma\eta_{t}\overline{\bm{P}}^{\pi^{\star}}\bm{\Delta}_{t-1}+\gamma\eta_{t}\left(\bm{\Lambda}\bm{P}_{\mathcal{K}}-\bm{P}\right)\bm{V}^{\star}.
\end{align*}
Repeatedly invoking these two recursive relations leads to 
\begin{align}
\bm{\Delta}_{t} & \leq\eta_{0}^{\left(t\right)}\bm{\Delta}_{0}+\sum_{i=1}^{t}\eta_{i}^{\left(t\right)}\gamma\left(\overline{\bm{P}}^{\pi_{t-1}}\bm{\Delta}_{t-1}+\bm{\Lambda}\left(\widehat{\bm{P}}_{\mathcal{K}}^{\left(t\right)}-\bm{P}_{\mathcal{K}}\right)\bm{V}_{t-1}+\left(\bm{\Lambda}\bm{P}_{\mathcal{K}}-\bm{P}\right)\bm{V}^{\star}\right),\label{eq:syn-upper}\\
\bm{\Delta}_{t} & \geq\eta_{0}^{\left(t\right)}\bm{\Delta}_{0}+\sum_{i=1}^{t}\eta_{i}^{\left(t\right)}\gamma\left(\overline{\bm{P}}^{\pi^{\star}}\bm{\Delta}_{t-1}+\bm{\Lambda}\left(\widehat{\bm{P}}_{\mathcal{K}}^{\left(t\right)}-\bm{P}_{\mathcal{K}}\right)\bm{V}_{t-1}+\left(\bm{\Lambda}\bm{P}_{\mathcal{K}}-\bm{P}\right)\bm{V}^{\star}\right),\label{eq:syn-lower}
\end{align}
where 
\[
\eta_{i}^{\left(t\right)}\coloneqq\begin{cases}
\prod_{j=1}^{t}\left(1-\eta_{j}\right), & \text{if }i=0,\\
\eta_{i}\prod_{j=i+1}^{t}\left(1-\eta_{j}\right), & \text{if }0<i<t,\\
\eta_{t}, & \text{if }i=t.
\end{cases}
\]
Here we adopt the same notations as \cite{li2021qlearning}.

To begin with, we consider the upper bound \eqref{eq:syn-upper}.
It can be further decomposed as 
\begin{align}
\bm{\Delta}_{t} & \leq\underbrace{\eta_{0}^{\left(t\right)}\bm{\Delta}_{0}+\sum_{i=1}^{\left(1-\alpha\right)t}\eta_{i}^{\left(t\right)}\gamma\left(\overline{\bm{P}}^{\pi_{t-1}}\bm{\Delta}_{t-1}+\bm{\Lambda}\left(\widehat{\bm{P}}_{\mathcal{K}}^{\left(t\right)}-\bm{P}_{\mathcal{K}}\right)\bm{V}_{t-1}\right)}_{\eqqcolon\bm{\theta}_{t}}+\underbrace{\sum_{i=\left(1-\alpha\right)t+1}^{t}\eta_{i}^{\left(t\right)}\gamma\bm{\Lambda}\left(\widehat{\bm{P}}_{\mathcal{K}}^{\left(t\right)}-\bm{P}_{\mathcal{K}}\right)\bm{V}_{i-1}}_{\eqqcolon\bm{\nu}_{t}}\nonumber \\
 & \qquad+\underbrace{\sum_{i=1}^{t}\eta_{i}^{\left(t\right)}\gamma\left(\bm{\Lambda}\bm{P}_{\mathcal{K}}-\bm{P}\right)\bm{V}^{\star}}_{\eqqcolon\bm{\omega}_{t}}+\sum_{i=\left(1-\alpha\right)t+1}^{t}\eta_{i}^{\left(t\right)}\gamma\overline{\bm{P}}^{\pi_{t-1}}\bm{\Delta}_{i-1},\label{eq:delta-decompose-further}
\end{align}
where we define $\alpha\coloneqq C_{4}(1-\gamma)/\log T$ for some
constant $C_{4}>0$. Next, we turn to bound $\bm{\theta}_{t}$ and
$\bm{\nu}_{t}$ respectively for any $t$ satisfying $\frac{T}{c_{2}\log\frac{1}{1-\gamma}}\leq t\leq T$
with stepsize choice \eqref{eq:syn-q-stepsize}.

\paragraph{Bounding $\bm{\omega}_{t}$.}
It is straightforward to bound
\begin{align*}
\left\Vert \bm{\omega}_{t}\right\Vert _{\infty} & \overset{\text{(i)}}{=}\left\Vert \gamma\left(\bm{\Lambda}\bm{P}_{\mathcal{K}}-\bm{P}\right)\bm{V}^{\star}\right\Vert _{\infty}\\
 & \overset{\text{(ii)}}{\leq}\gamma\left(\left\Vert \bm{\Lambda}\right\Vert _{1}\left\Vert \left(\bm{P}_{\mathcal{K}}-\widetilde{\bm{P}}_{\mathcal{K}}\right)\bm{V}^{\star}\right\Vert _{\infty}+\left\Vert \left(\widetilde{\bm{P}}-\bm{P}\right)\bm{V}^{\star}\right\Vert _{\infty}\right)\\
 & \overset{\text{(iii)}}{\leq}\frac{2\gamma\xi}{1-\gamma},
\end{align*}
where the first equality comes from the fact that $\sum_{i=1}^{t}\eta_{i}^{(t)}=1$
\cite[Equation (40)]{li2021qlearning}; the second inequality utilizes
$\widetilde{\bm{P}}=\bm{\Lambda}\widetilde{\bm{P}}_{\mathcal{K}}$;
the last line uses the facts that $\Vert\bm{\Lambda}\Vert_{1}=1$,
$\Vert\bm{V}^{\star}\Vert_{\infty}\leq(1-\gamma)^{-1}$ and $\Vert\widetilde{\bm{P}}_{\mathcal{K}}-\bm{P}_{\mathcal{K}}\Vert_{1}\leq\Vert\widetilde{\bm{P}}-\bm{P}\Vert_{1}\leq\xi$. 

\paragraph{Bounding $\bm{\theta}_{t}$.}

By similar derivation as Step 1 in \cite[Appendix A.2]{li2021qlearning},
we have 
\begin{align}
\left\Vert \bm{\theta}_{t}\right\Vert _{\infty} & \leq\eta_{0}^{\left(t\right)}\left\Vert \bm{\Delta}_{0}\right\Vert _{\infty}+t\max_{1\leq i\leq\left(1-\alpha\right)t}\eta_{i}^{\left(t\right)}\max_{1\leq i\leq\left(1-\alpha\right)t}\left(\left\Vert \overline{\bm{P}}^{\pi_{t-1}}\bm{\Delta}_{i-1}\right\Vert _{\infty}+\left\Vert \bm{\Lambda}\widehat{\bm{P}}_{\mathcal{K}}^{\left(t\right)}\bm{V}_{i-1}\right\Vert _{\infty}+\left\Vert \bm{\Lambda}\bm{P}_{\mathcal{K}}\bm{V}_{i-1}\right\Vert _{\infty}\right)\nonumber \\
 & \overset{\text{(i)}}{\leq}\eta_{0}^{\left(t\right)}\left\Vert \bm{\Delta}_{0}\right\Vert _{\infty}+t\max_{1\leq i\leq\left(1-\alpha\right)t}\eta_{i}^{\left(t\right)}\max_{1\leq i\leq\left(1-\alpha\right)t}\left(\left\Vert \bm{\Delta}_{i-1}\right\Vert _{\infty}+2\left\Vert \bm{V}_{i-1}\right\Vert _{\infty}\right)\nonumber \\
 & \overset{\text{(ii)}}{\leq}\frac{1}{2T^{2}}\cdot\frac{1}{1-\gamma}+\frac{1}{2T^{2}}\cdot t\cdot\frac{3}{1-\gamma}\nonumber \\
 & \leq\frac{2}{\left(1-\gamma\right)T},\label{eq:theta-t}
\end{align}
where (i) is due to the fact that $\Vert\overline{\bm{P}}^{\pi_{t-1}}\Vert_{1}=\Vert\bm{\Lambda}\widehat{\bm{P}}_{\mathcal{K}}^{\left(t\right)}\Vert_{1}=\Vert\bm{\Lambda}\bm{P}_{\mathcal{K}}\Vert_{1}=1$ and
(ii) comes from \cite[Equation (39a)]{li2021qlearning}.

\paragraph{Bounding $\bm{\nu}_{t}$.}

To control the second term, we apply the following Freedman's inequality.
\begin{lemma}[Freedman's Inequality]\label{lemma:freedman}
Consider a real-valued martingale $\{Y_{k}:k=0,1,2,\cdots\}$ with
difference sequence $\{X_{k}:k=1,2,3,\cdots\}$. Assume that the difference
sequence is uniformly bounded: 
\[
\left|X_{k}\right|\leq R\qquad\text{and}\qquad\mathbb{E}\left[X_{k}|\{X_j\}_{j=1}^{k-1}\right]=0\qquad\text{for all }k\geq1.
\]
Let 
\[
S_{n}\coloneqq\sum_{k=1}^{n}X_{i},\qquad T_{n}\coloneqq\sum_{k=1}^{n}\mathsf{Var}\left\{ X_{k}|\{X_j\}_{j=1}^{k-1}\right\} .
\]
Then for any given $\sigma^{2}\geq0$, one has 
\[
\mathbb{P}\left(\left|S_{n}\right|\geq\tau\text{ and }T_{n}\leq\sigma^{2}\right)\leq2\exp\left(-\frac{\tau^{2}/2}{\sigma^{2}+R\tau/3}\right).
\]

In addition, suppose that $W_{n}\leq\sigma^{2}$ holds deterministically.
For any positive integer $K\geq1$, with probability at least $1-\delta$
one has 
\[
\left|S_{n}\right|\leq\sqrt{8\max\left\{ T_{n},\frac{\sigma^{2}}{2^{K}}\right\} \log\frac{2K}{\delta}}+\frac{4}{3}R\log\frac{2K}{\delta}.
\]
\end{lemma}
\begin{proof}
	See \cite[Theorem 4]{li2021qlearning}.
\end{proof}To apply this inequality, we can express $\bm{\nu}_{t}$
as 
\[
\bm{\nu}_{t}\coloneqq\sum_{i=\left(1-\alpha\right)t+1}^{t}\bm{x}_{i},
\]
with 
\begin{equation}
\bm{x}_{i}\coloneqq\eta_{i}^{\left(t\right)}\gamma\bm{\Lambda}\left(\widehat{\bm{P}}_{\mathcal{K}}^{\left(t\right)}-\bm{P}_{\mathcal{K}}\right)\bm{V}_{i-1},\quad\text{and}\quad\mathbb{E}\left[\bm{x}_{i}|\bm{V}_{i-1},\cdots,\bm{V}_{0}\right]=\bm{0}.\label{eq:defn-xt}
\end{equation}

\begin{enumerate}
\item In order to calculate bound $R$ in Lemma \ref{lemma:freedman}, one
has 
\begin{align*}
B & \coloneqq\max_{\left(1-\alpha\right)t<t\leq t}\left\Vert \bm{x}_{i}\right\Vert _{\infty}\leq\max_{\left(1-\alpha\right)t<t\leq t}\left\Vert \eta_{i}^{\left(t\right)}\bm{\Lambda}\left(\widehat{\bm{P}}_{\mathcal{K}}^{\left(t\right)}-\bm{P}_{\mathcal{K}}\right)\bm{V}_{i-1}\right\Vert _{\infty}\\
 & \leq\max_{\left(1-\alpha\right)t<t\leq t}\eta_{i}^{\left(t\right)}\left(\left\Vert \bm{\Lambda}\widehat{\bm{P}}_{\mathcal{K}}^{\left(t\right)}\right\Vert _{1}+\left\Vert \bm{\Lambda}\bm{P}_{\mathcal{K}}\right\Vert _{1}\right)\left\Vert \bm{V}_{i-1}\right\Vert _{\infty}\\
 & \leq\max_{\left(1-\alpha\right)t<t\leq t}\eta_{i}^{\left(t\right)}\cdot\frac{2}{1-\gamma}\leq\frac{4\log^{4}T}{\left(1-\gamma\right)^{2}T},
\end{align*}
where the last inequality comes from \cite[Eqn (39b)]{li2021qlearning}
and the fact that $\Vert\bm{V}_{i-1}\Vert_{\infty}\leq\frac{1}{1-\gamma}$. 
\item Then regarding the variance term, we claim for the moment that
\begin{align}
\bm{W}_{t} & \coloneqq\sum_{i=\left(1-\alpha\right)t+1}^{t}\text{diag}\left(\mathsf{Var}\left(\bm{x}_{i}|\bm{V}_{i-1},\cdots,\bm{V}_{0}\right)\right)\nonumber\\
& \leq\gamma^{2}\sum_{i=\left(1-\alpha\right)t+1}^{t}\left(\eta_{i}^{\left(t\right)}\right)^{2}\mathsf{Var}_{\overline{\bm{P}}}\left(\bm{V}_{i-1}\right).\label{eq:claim-Wt}
\end{align}
Then we have
\begin{align}
\bm{W}_{t}  & \leq\max_{\left(1-\alpha\right)t\leq i\leq t}\eta_{i}^{\left(t\right)}\left(\sum_{i=\left(1-\alpha\right)t+1}^{t}\eta_{i}^{\left(t\right)}\right)\max_{\left(1-\alpha\right)t\leq i<t}\mathsf{Var}_{\overline{\bm{P}}}\left(\bm{V}_{i}\right)\nonumber \\
 & \leq\frac{2\log^{4}T}{\left(1-\gamma\right)T}\max_{\left(1-\alpha\right)t\leq i<t}\mathsf{Var}_{\overline{\bm{P}}}\left(\bm{V}_{i}\right),
\end{align}
where the second line comes from \cite[Eqns (39b), (40)]{li2021qlearning}.
A trivial upper bound for $\bm{W}_{t}$ is 
\[
\left|\bm{W}_{t}\right|\leq\frac{2\log^{4}T}{\left(1-\gamma\right)T}\cdot\frac{1}{\left(1-\gamma\right)^{2}}\bm{1}=\frac{2\log^{4}T}{\left(1-\gamma\right)^{3}T}\bm{1},
\]
which uses the fact that $\mathsf{Var}_{\bm{P}}(\bm{V}_{i})\leq\Vert\bm{V}_{i}\Vert_{\infty}^{2}\leq1/(1-\gamma)^{2}$. 
\end{enumerate}
Then, we invoke Lemma \ref{lemma:freedman} with $K=\left\lceil 2\log_{2}\frac{1}{1-\gamma}\right\rceil $ and apply the union bound argument
over $\mathcal{K}$ to arrive at 
\begin{align}
\left|\bm{\nu}_{t}\right| & \leq\sqrt{8\left(\bm{W}_{t}+\frac{\sigma^{2}}{2^{K}}\bm{1}\right)\log\frac{8KT\log\frac{1}{1-\gamma}}{\delta}}+\frac{4}{3}B\log\frac{8KT\log\frac{1}{1-\gamma}}{\delta}\bm{1}\nonumber \\
 & \leq\sqrt{8\left(\bm{W}_{t}+\frac{2\log^{4}T}{\left(1-\gamma\right)T}\bm{1}\right)\log\frac{8KT}{\delta}}+\frac{4}{3}B\log\frac{8KT\log\frac{1}{1-\gamma}}{\delta}\bm{1}\nonumber \\
 & \leq\sqrt{\frac{32\log^{4}T}{\left(1-\gamma\right)T}\log\frac{8KT}{\delta}\left(\max_{\left(1-\alpha\right)t\leq i<t}\mathsf{Var}_{\bm{\Lambda}\bm{P}_{\mathcal{K}}}\left(\bm{V}_{i}\right)+\bm{1}\right)}+\frac{12\log^{4}T}{\left(1-\gamma\right)^{2}T}\log\frac{8KT}{\delta}\bm{1}.\label{eq:nu-t}
\end{align}
Hence if we define 
\[
\bm{\varphi}_{t}\coloneqq64\frac{\log^{4}T\log\frac{KT}{\delta}}{\left(1-\gamma\right)T}\left(\max_{\frac{t}{2}\leq i\leq t}\mathsf{Var}_{\overline{\bm{P}}}\left(\bm{V}_{i}\right)+\bm{1}\right),
\]
then \eqref{eq:theta-t} and \eqref{eq:nu-t} implies that 
\begin{equation}
\left|\bm{\theta}_{t}\right|+\left|\bm{\nu}_{t}\right|+\left|\bm{\omega}_{t}\right|\leq\sqrt{\bm{\varphi}_{t}}+\frac{2\gamma\xi}{1-\gamma}\bm{1},\label{eq:phi-t}
\end{equation}
with probability over $1-\delta$ for all $2t/3\leq k\le t$, as long
as $T\gg\log^{4}T\log\frac{KT}{\delta}/\left(1-\gamma\right)^{3}$.
Therefore, plugging \eqref{eq:phi-t} into \eqref{eq:delta-decompose-further},
we arrive at the recursive relationship 
\begin{align*}
	\bm{\Delta}_{t}&\leq\sqrt{\bm{\varphi}_{t}}+\frac{2\gamma\xi}{1-\gamma}\bm{1}+\sum_{i=\left(1-\alpha\right)k+1}^{k}\eta_{i}^{\left(k\right)}\gamma\overline{\bm{P}}^{\pi_{i-1}}\bm{\Delta}_{i-1}=\sqrt{\bm{\varphi}_{t}}+\frac{2\gamma\xi}{1-\gamma}\bm{1}+\sum_{i=\left(1-\alpha\right)k}^{k-1}\eta_{i}^{\left(k\right)}\gamma\overline{\bm{P}}^{\pi_{i-1}}\bm{\Delta}_{i}.
\end{align*}
This recursion is expressed in a similar way as \cite[Eqn. (46)]{li2021qlearning}
so we can invoke similar derivation in \cite[Appendix A.2]{li2021qlearning}
to obtain that 
\begin{equation}
\bm{\Delta}_{t}\leq30\sqrt{\frac{\log^{4}T\log\frac{KT}{\delta}}{\left(1-\gamma\right)^{4}T}\left(1+\max_{\frac{t}{2}\leq i<t}\left\Vert \bm{\Delta}_{i}\right\Vert _{\infty}\right)}\bm{1}+\frac{2\gamma\xi}{\left(1-\gamma\right)^{2}}\bm{1}.\label{eq:delta-upper}
\end{equation}
Then we turn to \eqref{eq:syn-lower}. Applying a similar argument, 
we can deduce that 
\begin{equation}
\bm{\Delta}_{t}\geq-30\sqrt{\frac{\log^{4}T\log\frac{KT}{\delta}}{\left(1-\gamma\right)^{4}T}\left(1+\max_{\frac{t}{2}\leq i<t}\left\Vert \bm{\Delta}_{i}\right\Vert _{\infty}\right)}\bm{1}-\frac{2\gamma\xi}{\left(1-\gamma\right)^{2}}\bm{1}.\label{eq:delta-lower}
\end{equation}

For any $t$ satisfying $\frac{T}{c_{2}\log\frac{1}{1-\gamma}}\leq t\leq T$, taking \eqref{eq:delta-upper} and \eqref{eq:delta-lower} collectively
gives rise to 
\begin{equation}
	\left\Vert \bm{\Delta}_{t}\right\Vert _{\infty}\leq30\sqrt{\frac{\log^{4}T\log\frac{KT}{\delta}}{\left(1-\gamma\right)^{4}T}\left(1+\max_{\frac{t}{2}\leq i<t}\left\Vert \bm{\Delta}_{i}\right\Vert _{\infty}\right)}+\frac{2\gamma\xi}{\left(1-\gamma\right)^{2}}.\label{eq:Delta}
\end{equation}
Let 
\[
u_{k}\coloneqq\max\left\{ \left\Vert \bm{\Delta}_{t}\right\Vert _{\infty}:2^{k}\frac{T}{c_{2}\log\frac{1}{1-\gamma}}\leq t\leq T\right\} .
\]
By taking supremum over $t\in\{\lceil2^{k}T/(c_{2}\log\frac{1}{1-\gamma})\rceil,\ldots,T\}$
on both sides of (\ref{eq:Delta}), we have
\begin{equation}
	u_{k}\leq30\sqrt{\frac{\log^{4}T\log\frac{KT}{\delta}}{\left(1-\gamma\right)^{4}T}\left(1+u_{k-1}\right)}+\frac{2\gamma\xi}{\left(1-\gamma\right)^{2}}\qquad\forall\,\,1\leq k\leq \log\left(c_{2}\log\frac{1}{1-\gamma}\right).\label{eq:uk}
\end{equation}
It is straightforward to bound $u_{0}\leq\frac{1}{1-\gamma}$. For
$k\geq1$, it is straightforward to obtain from (\ref{eq:uk}) that
\begin{equation}
	u_{k}\leq3\max\left\{ 30\sqrt{\frac{\log^{4}T\log\frac{KT}{\delta}}{\left(1-\gamma\right)^{4}T}},30\sqrt{\frac{\log^{4}T\log\frac{KT}{\delta}}{\left(1-\gamma\right)^{4}T}u_{k-1}},\frac{2\gamma\xi}{\left(1-\gamma\right)^{2}}\right\}, \label{eq:uk-max}
\end{equation}
for $1\leq k\leq \log(c_{2}\log\frac{1}{1-\gamma})$.
We analyze (\ref{eq:uk-max}) under two different cases:
\begin{enumerate}
	\item If there exists some integer $k_{0}$ with $1\leq k_0<\lceil\log(c_{2}\log\frac{1}{1-\gamma})\rceil$, 
	such that 
	\[
	u_{k_{0}}\leq\max\left\{ 1,\frac{6\gamma\xi}{\left(1-\gamma\right)^{2}}\right\} ,
	\]
	then it is straightforward to check from (\ref{eq:uk-max}) that 
	\begin{equation}
		u_{k_{0}+1}\leq3\max\left\{ 30\sqrt{\frac{\log^{4}T\log\frac{KT}{\delta}}{\left(1-\gamma\right)^{4}T}},\frac{2\gamma\xi}{\left(1-\gamma\right)^{2}}\right\}  \label{eq:recursive}
	\end{equation}
	as long as $T\geq C_{3}(1-\gamma)^{-4}\log^{4}T\log(KT/\delta)$ for
	some sufficiently large constant $C_{3}>0$.
	\item Otherwise we have $u_{k}>\max\{1,\frac{6\gamma\xi}{(1-\gamma)^{2}}\}$
	for all $1\leq k< \lceil\log(c_{2}\log\frac{1}{1-\gamma})\rceil$. This together
	with (\ref{eq:uk-max}) suggests that \[\max\left\{1,\frac{6\gamma\xi}{(1-\gamma)^2}\right\}<3\max\left\{ 30\sqrt{\frac{\log^{4}T\log\frac{KT}{\delta}}{\left(1-\gamma\right)^{4}T}},30\sqrt{\frac{\log^{4}T\log\frac{KT}{\delta}}{\left(1-\gamma\right)^{4}T}u_{k-1}},\frac{2\gamma\xi}{\left(1-\gamma\right)^{2}}\right\},\]
	and therefore
	\[
	\max\left\{ 30\sqrt{\frac{\log^{4}T\log\frac{KT}{\delta}}{\left(1-\gamma\right)^{4}T}},30\sqrt{\frac{\log^{4}T\log\frac{KT}{\delta}}{\left(1-\gamma\right)^{4}T}u_{k-1}},\frac{2\gamma\xi}{\left(1-\gamma\right)^{2}}\right\} =30\sqrt{\frac{\log^{4}T\log\frac{KT}{\delta}}{\left(1-\gamma\right)^{4}T}u_{k-1}}
	\]
	for all $1\leq k\leq \log(c_{2}\log\frac{1}{1-\gamma})$. Let 
	\[
	v_{k}\coloneqq90\sqrt{\frac{\log^{4}T\log\frac{KT}{\delta}}{\left(1-\gamma\right)^{4}T}u_{k-1}}.
	\]
	Then we know from (\ref{eq:uk}) that
	\[
	u_{k}\leq v_{k}\qquad\forall\,\,1\leq k\leq \log\left(c_{2}\log\frac{1}{1-\gamma}\right).
	\]
	By applying the above two inequalities recursively, we know that
	\begin{align*}
		u_{k} & \leq v_{k}=\left(\frac{8100\log^{4}T\log\frac{KT}{\delta}}{\left(1-\gamma\right)^{4}T}\right)^{1/2}u_{k-1}^{1/2}\leq\left(\frac{8100\log^{4}T\log\frac{KT}{\delta}}{\left(1-\gamma\right)^{4}T}\right)^{1/2}v_{k-1}^{1/2}\\
		& \leq\left(\frac{8100\log^{4}T\log\frac{KT}{\delta}}{\left(1-\gamma\right)^{4}T}\right)^{1/2+1/4}u_{k-2}^{1/4}\leq\left(\frac{8100\log^{4}T\log\frac{KT}{\delta}}{\left(1-\gamma\right)^{4}T}\right)^{1/2+1/4}v_{k-2}^{1/4}\\
		& \leq\cdots\leq\left(\frac{8100\log^{4}T\log\frac{KT}{\delta}}{\left(1-\gamma\right)^{4}T}\right)^{1-1/2^{k}}u_{0}^{1/2^{k}}\leq\sqrt{\frac{8100\log^{4}T\log\frac{KT}{\delta}}{\left(1-\gamma\right)^{4}T}}\left(\frac{1}{1-\gamma}\right)^{1/2^{k}},
	\end{align*}
	where the last inequality holds as long as $T\geq C_{3}\log^{4}T\log(KT/\delta)(1-\gamma)^{-4}$ for some
	sufficiently large constant $C_{3}>0$. 
	Let $k_{0}=\widetilde{c}\log\log\frac{1}{1-\gamma}$ for some properly
	chosen constant $\widetilde{c}>0$ such that $k_{0}$ is an integer
	between $1$ and $\log(c_{2}\log\frac{1}{1-\gamma})$, we have
	\[
	u_{k_{0}}\leq\sqrt{\frac{8100\log^{4}T\log\frac{KT}{\delta}}{\left(1-\gamma\right)^{4}T}}\left(\frac{1}{1-\gamma}\right)^{1/2^{k_{0}}}=O\left(\sqrt{\frac{\log^{4}T\log\frac{KT}{\delta}}{\left(1-\gamma\right)^{4}T}}\right).
	\]
	When $T\geq C_{3}\log^{4}T\log(KT/\delta)(1-\gamma)^{-4}$ for some
	sufficiently large constant $C_{3}>0$, this implies that $u_{k_{0}}<1$,
	which contradicts with the preassumption that $u_{k}>\max\{1,\frac{6\gamma\xi}{(1-\gamma)^{2}}\}$
	for all $1\leq k\leq c_{2}\log\frac{1}{1-\gamma}$.
\end{enumerate}
Consequently, \eqref{eq:recursive} must hold true and then the definition of $u_k$ immediately leads to 
\begin{align*}
	\left\Vert \bm{\Delta}_{T}\right\Vert _{\infty}&\leq  90\sqrt{\frac{\log^{4}T\log\frac{KT}{\delta}}{\left(1-\gamma\right)^{4}T}}+\frac{6\gamma\xi}{\left(1-\gamma\right)^{2}}.
\end{align*}
Then for any $\varepsilon\in(0,1]$, one has 
\begin{align*}
\Vert\bm{\Delta}_{T}\Vert_{\infty} & \leq\varepsilon+\frac{6\gamma\xi}{\left(1-\gamma\right)^{2}},
\end{align*}
as long as 
\[
90\sqrt{\frac{\log^{4}T\log\frac{KT}{\delta}}{\left(1-\gamma\right)^{4}T}}\leq \varepsilon.
\]
Hence, if the total number of iterations $T$ satisfies 
\[
T\geq C_{3}\frac{\log^{4}T\log\frac{KT}{\delta}}{\left(1-\gamma\right)^{4}\varepsilon^{2}}
\]
for some sufficiently large constant $C_{3}>0$, \eqref{eq:model-free-Q-T}
would hold for Algorithm \ref{alg:model-based-rl} with probability
over $1-\delta$.

Finally, we are left to justify \eqref{eq:claim-Wt}. Recall the definition of $\bm{x}_{i}$ (cf.~\eqref{eq:defn-xt}),
one has 
\begin{align*}
\text{diag}\left(\mathsf{Var}\left(\bm{x}_{i}|\bm{V}_{i-1},\cdots,\bm{V}_{0}\right)\right) & =\gamma^{2}\left(\eta_{i}^{\left(t\right)}\right)^{2}\text{diag}\left(\mathsf{Var}\left(\bm{\Lambda}\left(\widehat{\bm{P}}_{\mathcal{K}}^{\left(t\right)}-\bm{P}_{\mathcal{K}}\right)\bm{V}_{i-1}|\bm{V}_{i-1}\right)\right)\\
 & =\gamma^{2}\left(\eta_{i}^{\left(t\right)}\right)^{2}\text{diag}\left(\bm{\Lambda}\mathsf{Var}\left(\left(\hat{\bm{P}}_{\mathcal{K}}^{\left(i\right)}-\bm{P}_{\mathcal{K}}\right)\bm{V}_{i-1}|\bm{V}_{i-1}\right)\bm{\Lambda}^{\top}\right)\\
 & =\gamma^{2}\left(\eta_{i}^{\left(t\right)}\right)^{2}\left\{ \bm{\lambda}\left(s,a\right)^{2}\mathsf{Var}_{\bm{P}_{\mathcal{K}}}\left(\bm{V}_{i-1}\right)\right\} _{s,a},
\end{align*}
where the notation $\mathsf{Var}_{\bm{P}_{\mathcal{K}}}(\bm{V}_{i-1})$
is defined in \eqref{eq:defn-varpv}. Plugging this into the definition
of $\bm{W}_{t}$ leads to
\begin{align}
\bm{W}_{t} & =\gamma^{2}\sum_{i=\left(1-\alpha\right)t+1}^{t}\left(\eta_{i}^{\left(t\right)}\right)^{2}\left\{ \bm{\lambda}\left(s,a\right)^{2}\mathsf{Var}_{\bm{P}_{\mathcal{K}}}\left(\bm{V}_{i-1}\right)\right\} _{s,a}\nonumber\\
 & =\gamma^{2}\sum_{i=\left(1-\alpha\right)t+1}^{t}\left(\eta_{i}^{\left(t\right)}\right)^{2}\left\{ \bm{\lambda}\left(s,a\right)^{2}\left(\bm{P}_{\mathcal{K}}\left(\bm{V}_{i-1}\circ\bm{V}_{i-1}\right)-\left(\bm{P}_{\mathcal{K}}\bm{V}_{i-1}\right)\circ\left(\bm{P}_{\mathcal{K}}\bm{V}_{i-1}\right)\right)\right\} _{s,a}.\label{w-claim}
\end{align}
Then we introduce a useful claim as follows. The proof is deferred to Appendix \ref{subsec:proof-claim}. 
\begin{claim}\label{claim:1}For any state-action pair $(s,a)\in\mathcal{S}\times\mathcal{A}$ and vector $\bm{V}\in\mathbb{R}^{|\mathcal{S}|}$, one has 
	\begin{align}
	& \bm{\lambda}\left(s,a\right)^{2}\left(\bm{P}_{\mathcal{K}}\left(\bm{V}\circ\bm{V}\right)-\left(\bm{P}_{\mathcal{K}}\bm{V}\right)\circ\left(\bm{P}_{\mathcal{K}}\bm{V}\right)\right)\nonumber \\
	& \quad\leq\bm{\lambda}\left(s,a\right)\bm{P}_{\mathcal{K}}\left(\bm{V}\circ\bm{V}\right)-\left(\bm{\lambda}\left(s,a\right)\bm{P}_{\mathcal{K}}\bm{V}\right)\circ\left(\bm{\lambda}\left(s,a\right)\bm{P}_{\mathcal{K}}\bm{V}\right).\label{eq:syn-claim}
	\end{align}
\end{claim}
By invoking this claim with $\bm{V}=\bm{V}^{i-1}$ and taking collectively with \eqref{w-claim}, one has
\begin{align*}
\bm{W}_{t} & \leq\gamma^{2}\sum_{i=\left(1-\beta\right)t+1}^{t}\left(\eta_{i}^{\left(t\right)}\right)^{2}\left\{ \bm{\lambda}\left(s,a\right)\bm{P}_{\mathcal{K}}\left(\bm{V}_{i-1}\circ\bm{V}_{i-1}\right)-\left(\bm{\lambda}\left(s,a\right)\bm{P}_{\mathcal{K}}\bm{V}_{i-1}\right)\circ\left(\bm{\lambda}\left(s,a\right)\bm{P}_{\mathcal{K}}\bm{V}_{i-1}\right)\right\} _{s,a}\\
 & =\gamma^{2}\sum_{i=\left(1-\beta\right)t+1}^{t}\left(\eta_{i}^{\left(t\right)}\right)^{2}\left[\bm{\Lambda}\bm{P}_{\mathcal{K}}\left(\bm{V}_{i-1}\circ\bm{V}_{i-1}\right)-\left(\bm{\Lambda}\bm{P}_{\mathcal{K}}\bm{V}_{i-1}\right)\circ\left(\bm{\Lambda}\bm{P}_{\mathcal{K}}\bm{V}_{i-1}\right)\right]\\
 & =\gamma^{2}\sum_{i=\left(1-\beta\right)t+1}^{t}\left(\eta_{i}^{\left(t\right)}\right)^{2}\mathsf{Var}_{\overline{\bm{P}}}\left(\bm{V}_{i-1}\right),
\end{align*}
which is the desired result.

\subsection{Proof of Claim \ref{claim:1}}\label{subsec:proof-claim}
To simplify
notations in this proof, we use $[\lambda_{i}]_{i=1}^{K}$, $[P_{i,j}]_{1\leq i\leq K,1\leq j\leq|\mathcal{S}|}$
and $[V_{i}]_{i=1}^{\left|\mathcal{S}\right|}$ to denote $\bm{\lambda}(s,a)$,
$\bm{P}_{\mathcal{K}}$ and $\bm{V}$ respectively. Then one
has 
\begin{align*}
 & \bm{\lambda}\left(s,a\right)\bm{P}_{\mathcal{K}}\left(\bm{V}\circ\bm{V}\right)-\left(\bm{\lambda}\left(s,a\right)\bm{P}_{\mathcal{K}}\bm{V}\right)\circ\left(\bm{\lambda}\left(s,a\right)\bm{P}_{\mathcal{K}}\bm{V}\right)\\
 & \qquad-\bm{\lambda}\left(s,a\right)^{2}\left(\bm{P}_{\mathcal{K}}\left(\bm{V}\circ\bm{V}\right)-\left(\bm{P}_{\mathcal{K}}\bm{V}\right)\circ\left(\bm{P}_{\mathcal{K}}\bm{V}\right)\right)\\
 & \quad=\sum_{i=1}^{K}\sum_{j=1}^{\left|\mathcal{S}\right|}\lambda_{i}P_{i,j}V_{j}^{2}-\left(\sum_{i=1}^{K}\sum_{j=1}^{\left|\mathcal{S}\right|}\lambda_{i}P_{i,j}V_{j}\right)^{2}-\sum_{i=1}^{K}\sum_{j=1}^{\left|\mathcal{S}\right|}\lambda_{i}^{2}P_{i,j}V_{j}^{2}+\sum_{i=1}^{K}\lambda_{i}^{2}\left(\sum_{j=1}^{\left|\mathcal{S}\right|}P_{i,j}V_{j}\right)^{2}\\
 & \quad=\sum_{i=1}^{K}\sum_{j=1}^{\left|\mathcal{S}\right|}\lambda_{i}P_{i,j}V_{j}\left[\left(1-\lambda_{i}\right)V_{j}-\sum_{i'\neq i}\sum_{j'=1}^{\left|\mathcal{S}\right|}\lambda_{i'}P_{i',j'}V_{j'}\right].\\
 & \quad=\sum_{i=1}^{K}\sum_{j=1}^{\left|\mathcal{S}\right|}\lambda_{i}P_{i,j}V_{j}\left[\left(\sum_{i'=1}^{K}\sum_{j'=1}^{\left|\mathcal{S}\right|}\lambda_{i'}P_{i',j'}-\lambda_{i}\right)V_{j}-\sum_{i'\neq i}\sum_{j'=1}^{\left|\mathcal{S}\right|}\lambda_{i'}P_{i',j'}V_{j'}\right]\\
 & \quad=\sum_{i=1}^{K}\sum_{j=1}^{\left|\mathcal{S}\right|}\sum_{i'\neq i}\sum_{j'=1}^{\left|\mathcal{S}\right|}\lambda_{i}P_{i,j}V_{j}\lambda_{i'}P_{i',j'}\left(V_{j}-V_{j'}\right)
\end{align*}
where in the penultimate equality, we use the fact that 
\[
\sum_{i'=1}^{K}\sum_{j'=1}^{\left|\mathcal{S}\right|}\lambda_{i'}P_{i',j'}=\bm{\lambda}\left(s,a\right)\bm{P}_{\mathcal{K}}\bm{1}=1.
\]
It follows that 
\begin{align*}
 & \bm{\lambda}\left(s,a\right)\bm{P}_{\mathcal{K}}\left(\bm{V}\circ\bm{V}\right)-\left(\bm{\lambda}\left(s,a\right)\bm{P}_{\mathcal{K}}\bm{V}\right)\circ\left(\bm{\lambda}\left(s,a\right)\bm{P}_{\mathcal{K}}\bm{V}\right)\\
 & \qquad-\bm{\lambda}\left(s,a\right)^{2}\left(\bm{P}_{\mathcal{K}}\left(\bm{V}\circ\bm{V}\right)-\left(\bm{P}_{\mathcal{K}}\bm{V}\right)\circ\left(\bm{P}_{\mathcal{K}}\bm{V}\right)\right)\\
 & \quad=\sum_{i=1}^{K}\sum_{1\leq i'<i}\sum_{j=1}^{\left|\mathcal{S}\right|}\sum_{j'=1}^{\left|\mathcal{S}\right|}\left[\lambda_{i}P_{i,j}V_{j}\lambda_{i'}P_{i',j'}\left(V_{j}-V_{j'}\right)+\lambda_{i'}P_{i',j}V_{j}\lambda_{i}P_{i,j'}\left(V_{j}-V_{j'}\right)\right]\\
 & \quad=\sum_{i=1}^{K}\sum_{1\leq i'<i}\lambda_{i}\lambda_{i'}\left[\sum_{j=1}^{\left|\mathcal{S}\right|}\sum_{j'=1}^{\left|\mathcal{S}\right|}P_{i,j}V_{j}P_{i',j'}\left(V_{j}-V_{j'}\right)+\sum_{j=1}^{\left|\mathcal{S}\right|}\sum_{j'=1}^{\left|\mathcal{S}\right|}P_{i',j}V_{j}P_{i,j'}\left(V_{j}-V_{j'}\right)\right]\\
 & \quad\overset{\text{(i)}}{=}\sum_{i=1}^{K}\sum_{1\leq i'<i}\lambda_{i}\lambda_{i'}\left[\sum_{j=1}^{\left|\mathcal{S}\right|}\sum_{j'=1}^{\left|\mathcal{S}\right|}P_{i,j}V_{j}P_{i',j'}\left(V_{j}-V_{j'}\right)+\sum_{j=1}^{\left|\mathcal{S}\right|}\sum_{j'=1}^{\left|\mathcal{S}\right|}P_{i',j'}V_{j'}P_{i,j}\left(V_{j'}-V_{j}\right)\right]\\
 & \quad=\sum_{i=1}^{K}\sum_{1\leq i'<i}\lambda_{i}\lambda_{i'}\left[\sum_{j=1}^{\left|\mathcal{S}\right|}\sum_{j'=1}^{\left|\mathcal{S}\right|}P_{i,j}P_{i',j'}\left(V_{j}-V_{j'}\right)^{2}\right]\\
 & \quad\geq0,
\end{align*}
where in (i), we exchange the indices $j$ and $j'$.

\section{Feature dimension and the number of anchor state-action pairs} \label{appendix:dimension-number}
The assumption that the feature dimension (denoted by $K_{\mathsf{d}}$)
and the number of anchor state-action pairs (denoted by $K_{\mathsf{n}}$)
are equal is actually non-essential. In what follows, we will show
that if $K_{\mathsf{d}}\neq K_{\mathsf{n}}$, then we can modify the
current feature mapping $\phi:\mathcal{S}\times\mathcal{A}\to\mathbb{R}^{K_{\mathsf{d}}}$
to achieve a new feature mapping $\phi':\mathcal{S}\times\mathcal{A}\to\mathbb{R}^{K_{\mathsf{n}}}$
that does not change the transition model $P$. By doing so, the new
feature dimension $K_{\mathsf{n}}$ equals to the number of anchor
state-action pairs.

To begin with, we recall from Definition \ref{def:linear-model} that there exists $K_{\mathsf{d}}$
unknown functions $\psi_{1}$, $\cdots$, $\psi_{K_{\mathsf{d}}}:\mathcal{S}\rightarrow\mathbb{R}$,
such that 
\[
P\left(s'|s,a\right)=\sum_{k=1}^{K_{\mathsf{d}}}\phi_{k}\left(s,a\right)\psi_{k}\left(s'\right),
\]
for every $(s,a)\in\mathcal{S}\times\mathcal{A}$ and $s'\in\mathcal{S}$.
In addition, we also recall from Assumption \ref{assumption:anchor} that there exists $\mathcal{K}\subseteq\mathcal{S}\times\mathcal{A}$
with $\vert\mathcal{K}\vert=K_{\mathsf{n}}$ such that for any $(s,a)\in\mathcal{S}\times\mathcal{A}$,
\[
\phi\left(s,a\right)=\sum_{i:(s_{i},a_{i})\in\mathcal{K}}\lambda_{i}\left(s,a\right)\phi\left(s_{i},a_{i}\right)\in\mathbb{R}^{K_{\mathsf{d}}}\quad\text{for}\quad\sum_{i=1}^{K_{\mathsf{n}}}\lambda_{i}\left(s,a\right)=1\quad\text{and}\quad\lambda_{i}\left(s,a\right)\geq0.
\]

\textbf{Case 1:} $K_{\mathsf{d}}>K_{\mathsf{n}}$. In this case, the
vectors in $\{\phi(s,a):(s,a)\in\mathcal{K}\}$ are linearly independent.
For ease of presentation and without loss of generality, we assume
that $K_{\mathsf{d}}=K_{\mathsf{n}}+1$. This indicates that the matrix
$\bm{\Phi}\in\mathbb{R}^{K_{\mathsf{d}}\times(\vert\mathcal{S}\vert\vert\mathcal{A}\vert)}$
whose columns are composed of the feature vectors of all state-action
pairs has rank $K_{\mathsf{n}}$ and is hence not full row rank. This
suggests that there exists $K_{\mathsf{n}}$ linearly independent
rows (without loss of generality, we assume they are the first $K_{\mathsf{n}}$
rows). We can remove the last row from $\bm{\Phi}$ to obtain $\bm{\Phi}'\coloneqq\bm{\Phi}_{1:K_{\mathsf{n}},:}\in\mathbb{R}^{K_{\mathsf{n}}\times(\vert\mathcal{S}\vert\vert\mathcal{A}\vert)}$
such that $\bm{\Phi}'$ is full row rank. Then we show that we can
actually use the columns of $\bm{\Phi}'$ as new feature mappings.
To see why this is true, note that the last row $\Phi_{K_{\mathsf{n}}+1,:}$
can be represented as a linear combination of the first $K_{\mathsf{n}}$
rows, namely there must exist constants $\{c_{k}\}_{k=1}^{K_{\mathsf{n}}}$
such that for any $(s,a)\in\mathcal{S}\times\mathcal{A}$,
\[
\phi_{K_{\mathsf{n}}+1}(s,a)=\sum_{k=1}^{K_{\mathsf{n}}}c_{k}\phi_{k}(s,a).
\]
Define $\psi_{k}'=\psi_{k}+c_{k}\psi_{K_{\mathsf{n}}+1}$ for $k=1,\ldots,K_{\mathsf{n}}$,
we have
\begin{align*}
P\left(s'\vert s,a\right) & =\sum_{k=1}^{K_{\mathsf{d}}}\phi_{k}\left(s,a\right)\psi_{k}\left(s'\right)=\phi_{K_{\mathsf{n}}+1}\left(s,a\right)\psi_{K_{\mathsf{n}}+1}\left(s'\right)+\sum_{k=1}^{K_{\mathsf{n}}}\phi_{k}\left(s,a\right)\psi_{k}\left(s'\right)\\
& =\sum_{k=1}^{K_{\mathsf{n}}}\phi_{k}\left(s,a\right)\left[\psi_{k}\left(s'\right)+c_{k}\psi_{K_{\mathsf{n}}+1}\left(s'\right)\right]=\sum_{k=1}^{K_{\mathsf{n}}}\phi_{k}\left(s,a\right)\psi_{k}'\left(s'\right),
\end{align*}
which is linear with respect to the new $K_{\mathsf{n}}$ dimensional
feature vectors. It is also straightforward to check that the new
feature mapping satisfies Assumption 1 with the original anchor state-action
pairs $\mathcal{K}$.

\textbf{Case 2:} $K_{\mathsf{d}}<K_{\mathsf{n}}$. For ease of presentation
and without loss of generality, we assume that $K_{\mathsf{n}}=K_{\mathsf{d}}+1$
and that the subspace spanned by the feature vectors of anchor state-action
pairs is non-degenerate, i.e., has rank $K_{\mathsf{d}}$ (otherwise
we can use similar method as in Case 1 to further reduce the feature
dimension $K_{\mathsf{d}}$). In this case, the matrix $\bm{\Phi}_{\mathcal{K}}\in\mathbb{R}^{K_{\mathsf{d}}\times K_{\mathsf{n}}}$
whose columns are composed of the feature vectors of anchor state-action
pairs has rank $K_{\mathsf{d}}$. We can add $K_{\mathsf{n}}-K_{\mathsf{d}}=1$
new row to $\bm{\Phi}_{\mathcal{K}}$ to obtain $\bm{\Phi}_{\mathcal{K}}'\in\mathbb{R}^{K_{\mathsf{n}}\times K_{\mathsf{n}}}$
such that $\bm{\Phi}_{\mathcal{K}}'$ has full rank $K_{\mathsf{n}}$.
Then we let the columns of $\bm{\Phi}_{\mathcal{K}}'=[\phi'(s,a)]_{(s,a)\in\mathcal{K}}$
to be the new feature vectors of the anchor state-action pairs, and
define the new feature vectors for all other state-action pairs $(s,a)\notin\mathcal{K}$
by 
\[
\phi'\left(s,a\right)=\sum_{i:(s_{i},a_{i})\in\mathcal{K}}\lambda_{i}\left(s,a\right)\phi'\left(s_{i},a_{i}\right).
\]
We can check that the transition model $P$ is not changed if we let
$\psi_{K_{\mathsf{n}}}(s')=0$ for every $s'\in\mathcal{S}$. It is
also straightforward to check that Assumption \ref{assumption:anchor} is satisfied.

To conclude, when $K_{\mathsf{d}}\neq K_{\mathsf{n}}$, we can always
construct a new set of feature mappings with dimension $K_{\mathsf{n}}$
such that: (i) the feature dimension equals to the number of anchor
state-action pairs (they are both $K_{\mathsf{n}}$); (ii) the transition
model can still be linearly parameterized by this new set of feature
mappings; and (iii) the anchor state-action pair assumption (Assumption
\ref{assumption:anchor}) is satisfied with the original anchor state-action pairs.

\bibliographystyle{alpha}
\bibliography{bibfile}

\end{document}